\documentclass[10pt]{article}

\usepackage[dvipsnames]{xcolor}
\usepackage[preprint]{rlj}
\usepackage{amssymb}            
\usepackage{mathtools}          
\usepackage{mathrsfs}           
\usepackage{graphicx}           
\usepackage{subcaption}         
\usepackage[space]{grffile}     
\usepackage{url}                
\usepackage{lipsum}             
\usepackage{graphicx}
\usepackage{amsmath,amsfonts}
\usepackage{bm}
\usepackage{amsthm}

\newcommand{\dist}[1]{\normalfont{\textrm{#1}}}

\def\equationautorefname#1#2\null{Eq.#1(#2\null)}

\newtheorem{definition}{Definition}
\newtheorem{theorem}{Theorem}
\newtheorem{lemma}{Lemma}
\newtheorem{proposition}{Proposition}
\newtheorem{corollary}{Corollary}

\hypersetup{
  linkcolor  = BrickRed,
  citecolor  = MidnightBlue,
  urlcolor   = Aquamarine,
  colorlinks = true,
}

\usepackage{tikz} 
\usetikzlibrary{shapes.geometric,arrows,arrows.meta} 
\usetikzlibrary{shapes.multipart}
\usetikzlibrary{calc}
\usetikzlibrary{positioning}
\tikzset{%
  >={Latex[width=2mm,length=2mm]},
  base/.style = {rectangle, rounded corners, 
                 minimum width=1.6cm, minimum height=0.8cm, 
                 text centered},
  start/.style = {base, draw=black, fill=orange!30},
  stop/.style = {base, draw=black, fill=red!30},
  process/.style = {base, draw=black, fill=orange!15,},
  criteria/.style = {ellipse, text centered, fill=Aquamarine!40, inner sep=0.05cm},
}

\title{AI in a vat: Fundamental limits of efficient world\\ modelling for agent sandboxing and interpretability}
\setrunningtitle{AI in a vat}

\author{Fernando E. Rosas\textsuperscript{1-5}, Alexander Boyd\textsuperscript{1,6}, Manuel Baltieri\textsuperscript{7,1}}
\emails{\small{f.rosas@sussex.ac.uk, alecboy@gmail.com, manuel\_baltieri@araya.org}}

\affiliations{
$^{1}$\textbf{Department of Informatics, University of Sussex}\\
$^{2}$\textbf{Sussex AI and Sussex Centre for Consciousness Science, University of Sussex}\\
$^{3}$\textbf{Centre for Complexity Science and Center for Psychedelic Research, Department of Brain Sciences, Imperial College London}\\
$^{4}$\textbf{Center for Eudaimonia and Human Flourishing, University of Oxford}\\
$^{5}$\textbf{Principles of Intelligent Behaviour in Biological and Social Systems (PIBBSS)}\\
$^{6}$\textbf{Beyond Institute for Theoretical Science (BITS)}\\
$^{7}$\textbf{Araya Inc.}
}

\contribution{
    This paper conceptualises and formalises a novel problem: building efficient world models for an operator to sandbox, evaluate, and interpret AI agents before deployment. 
    }
    {
    Prior work (e.g.~\citep{ha2018recurrent,hafner2019dream}) focuses on world models from the perspective of the agent using for boosting performance, and has not considered this safety-inspired perspective.
    }

\contribution{
    We introduce generalised transducers based on quasi-probabilities, leading to a more efficient approach to compress world models at the expense of their interpretability.
    }
    {
    Generalised transducers are an extension of generalised hidden Markov models, which have been thoroughly studied in previous works~\citep{upper1997theory,vidyasagar2011complete}.
    }

\contribution{
    We provide a unifying framework to investigate and reason about world models of beliefs, and show that all  models that can be calculated by an agent in real time can be bisimulated into a canonical world model known as \emph{$\epsilon$-transducer}.
    }
    {
    The minimality of the $\epsilon$-transducer among prescient rival partitions was proven in~\citep{barnett2015computational}, without investigating links with bisimulation or other concepts from reinforcement learning. Relationships between bisimulation and other computational mechanics constructions were investigated by~\cite{zhang2019learning}.
    }

\contribution{
    We introduce the notion of \emph{reverse} interpretability, which is related to retrodictive analyses that can identify the roots of undesirable outcomes.  
    }
    {
    Standard interpretability approaches assess agents with respect to their capabilities to predict and plan with respect to future events~\citep{nanda2023emergent,gurnee2023language,shai2025transformers}.
    }

\contribution{
    We introduce the notion of reversible transducer, and identify necessary and sufficient conditions for it. We also introduce and explore the notion of retrodictive beliefs.
    }
    {
    Retrodictive and reversible hidden Markov models have been investigated  by~\cite{ellison2009prediction,ellison2011information}.
    }


\summary{
\noindent
While traditionally conceived as tools for model-based reinforcement learning agents to improve their task performance, recent works have proposed \emph{world models} as a way to build controlled virtual environments where AI agents can be thoroughly evaluated before deployment. 
However, the efficacy of these approaches critically rely on the ability of world models to accurately represent real environments, which can results on high computational costs that may substantially restrict testing capabilities. 
Drawing inspiration from the `brain in a vat' thought experiment, here we investigate methods to simplify world models that remain agnostic to the agent under evaluation. 
Our results reveal a fundamental trade-off inherent to the construction of world models related to their efficiency and interpretability. 
Furthermore, we develop approaches that either minimise memory usage, establish the limits on what is learnable, or enable retrodictive analyses tracking the causes of undesirable outcomes. 
These results sheds light on the fundamental constraints that shape the design space of world modelling for agent sandboxing and interpretability.
}

\begin{document}

\makeCover  
\maketitle  

\begin{abstract}
Recent work proposes using world models to generate controlled virtual environments in which AI agents can be tested before deployment to ensure their reliability and safety. 
However, accurate world models often have high computational demands that can severely restrict the scope and depth of such assessments. 
Inspired by the classic `brain in a vat' thought experiment, here we investigate ways of simplifying world models that remain agnostic to the AI agent under evaluation. 
By following principles from computational mechanics, our approach reveals a fundamental trade-off in world model construction between efficiency and interpretability, demonstrating that no single world model can optimise all desirable characteristics. 
Building on this trade-off, we identify procedures to build world models that either minimise memory requirements, delineate the boundaries of what is learnable, or allow tracking causes of undesirable outcomes. 
In doing so, this work establishes fundamental limits in world modelling, leading to actionable guidelines that inform core design choices related to effective agent evaluation.
\end{abstract}

\section{Introduction}

Breakthroughs in deep reinforcement learning are progressively enabling AI agents capable of mastering complex tasks across a wide array of domains~\citep{arulkumaran2017deep, wang2022deep}, and a new generation of agents leveraging large language models~\citep{wang2024survey} and large multimodal models~\citep{yin2024survey} is expected to drive a new wave of technological innovation with the potential to benefit all sectors of the global economy~\citep{WEF2025}.
Alongside these benefits, the proliferation of increasingly advanced autonomous AI systems will also bring important new risks regarding their safety, controllability, and alignment with human values~\citep{bengio2024international,tang2024prioritizing}. 
Given these far-reaching prospects, it is imperative to develop frameworks and methodologies to guarantee the safe and beneficial integration of these technologies to our societies.

One path to pursue AI safety and alignment is to use synthetic world models as sandbox environments to evaluate AI agents without real-world consequences~\citep{dalrymple2024towards,diaz2023connecting}. 
These simulated environments are ideal for observing how AI agents handle edge cases and respond to novel situations, potentially revealing safety issues or alignment failures before deployment~\citep{he2024security}. 
The efficacy of this approach, however, critically relies on the world model accurately representing relevant aspects of real environments, which is key for guaranteeing that the agent's behaviour in simulation may transfer to real-world settings. 
Thus, a key challenge lies in dealing with the computational demands of high-fidelity simulations, whose costs can impose heavy restrictions on the breadth and depth of safety and reliability assessments.

Here we address these issues by investigating the fundamental limits that shape the design of world models. 
By bridging concepts from reinforcement learning, control theory, and computational mechanics, we identify a fundamental trade-off between the computational efficiency of a world model and its interpretability. 
This approach also leads to the distinction between \emph{forward} and \emph{reverse} interpretability approaches, where the former characterises the predictive capabilities of agents, and the latter enables retrodictive analyses of the origins of undesirable outcomes. 
Overall, this work establishes foundational groundwork that leads to actionable guidelines for building world models to study AI agents following different desiderata.

\section{Scenario and approach}
\label{sec:scenario_and_approach}

\begin{quotation}
    \noindent \textit{Representation and what is represented belong to two completely different worlds.} 
\end{quotation}
\begin{flushright}
    Hermann von Helmholtz, \textit{Handbuch der physiologischen Optik} (1867)
\end{flushright}

Consider the design of a world model to sandbox and test an AI agent~\citep{dalrymple2024towards}. 
What should this world model look like? What information should it encode? And for what purpose?

To ensure a reliable assessment of AI agent behaviour from simulations to a real-world setting, world models must faithfully reflect the world’s structure and dynamics.
This could be seen as suggesting that designing reliable world models is critically bounded by a trade-off between accuracy and computational tractability. 
Interestingly, this trade-off can be partially circumvented by recognising that effective world models only need to incorporate variables that make a difference for the AI's actions, and these variables only require a granularity that is sufficient to simulate their dynamics.

To illustrate this idea, consider how one could choose to build a world model to sandbox a simple robot. 
Although one could in principle design a simulation that includes the quantum dynamics of the whole planet, such a simulation would not only be computationally unfeasible but also unnecessary to answer most questions of interest. 
Indeed, such a world model would likely be too spatially extended (by including regions of the planet that are inaccessible to the robot) and have too much resolution (by including quantum effects for a fundamentally classical robot).
To avoid this, one could instead design a more parsimonious world model that factors out indistinguishable properties from the robot's perspective, focusing on the agent's `interface' including, for instance, sensorimotor contingencies~\citep{o2001sensorimotor, baltieri2017active, baltieri2019generative, tschantz2020learning, mannella2021active} or task-relevant information~\citep{zhang2021learning}.

Related questions have been extensively investigated in the philosophy of mind and cognitive (neuro)science literature, and more recently in reinforcement learning. 
These works suggest an important insight: while an agent's actions turn into outcomes through the mediation of the external world, the agent lacks direct access to the world's `true nature' and only interacts with it via its inputs and outputs~\citep{clark2013whatever,seth2018being}. 
This notion is illustrated by the classical \emph{`brain in a vat'} thought experiment, which proposes that if an organism's brain were to be placed in a vat, and a computer used to read the brain's output signals and generate plausible sensory signals, then the brain may not be able to tell it is in fact in a vat.\footnote{The modern form of this thought experiment is due to~\cite{putnam1981reason}, but has roots in Descartes' `evil demon'~\citep{descartes1641} and Plato's cave allegory \citep{plato375} while serving as inspiration for popular media such as \textit{The Matrix} movies.} 
This thought experiment suggests that an ideal world model should depend only on three elements: the set of possible actions of the agent $\mathcal{A}$, the set of possible outcomes affecting the agent $\mathcal{Y}$,\footnote{The outcome may be a combination of a quantity observable by the agent and a reward signal, so that $\mathcal{Y} = \mathcal{O}\times \mathbb{R}$.} and the statistical relationship between action sequences and outcomes. 
In fact, it should be possible --- at least in principle --- to build a compressed representation of the `effective world' of an AI agent that cannot be distinguished from a full simulation, irrespectively of how smart or powerful it may be.

These ideas can be operationalised using principles from computational mechanics~\citep{crutchfield1994calculi,crutchfield2012between}, which reveal how observable processes can be generated by multiple data-generating procedures (see~\autoref{app:example} for an example). 
Embracing this multiplicity leads to a perspective that we describe as \emph{`AI in a vat'}, which posits that designers should not focus on a single world model, but instead should (i) consider the class of all world models that are indistinguishable from the AI agent's perspective, (ii) characterise their properties, and then (iii) choose one depending on specific priorities. 
After setting some formal foundations in \autoref{sec:formal_foundations}, the remainder of this work explores how this approach reveals key design choices and procedures to construct optimal models related to different ways of using world models to pursue AI safety and alignment (see~\autoref{fig:diagram}):
\begin{itemize}
    \item \textit{Computational efficiency} (\autoref{sec:reduction}): sandboxing agents to evaluate their behaviour or provide formal guarantees about their capabilities~\citep{dalrymple2024towards} using a minimal amount of resources.
    \item \textit{Forward interpretability} (\autoref{sec:epistemic_worlds}): building models to study which features are learnable by agents, and how representations are encoded inside them~\citep{shai2025transformers}.
    \item \textit{Reverse interpretability} (\autoref{sec:backwards_interp}): deploying models that can be run backward in time to investigate the origins and tipping points that lead to specific --- desirable or undesirable --- outcomes.
\end{itemize}

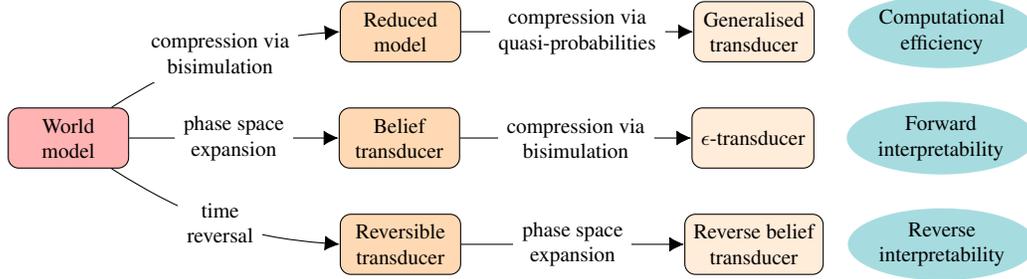
\begin{figure}[ht]
\centering
\begin{tikzpicture}[
    node distance=2cm, 
    every node/.style={fill=white,font=\footnotesize, minimum width=1em}, 
    align=center
    ]
  
  \node (world) [stop] {World\\model};
  
  \node (reduced) [start, above right of=world, xshift=3.0cm] {Reduced\\model};
  \node (predictive) [start, right of=world, xshift=2.4cm] {Belief\\transducer};
  \node (retro) [start, below right of=world, xshift=3.0cm] {Reversible\\transducer};
  
  \node (gentrans) [process, right of=reduced, xshift=2.7cm] {Generalised\\transducer};
  \node (etrans) [process, right of=predictive, xshift=2.7cm] {$\epsilon$-transducer};
  \node (rev_beliefs) [process, right of=retro, xshift=2.7cm] {Reverse belief\\transducer};
  
  \node (efficiency) [criteria, right of=gentrans, xshift=0.5cm] {Computational\\efficiency};
  \node (forward) [criteria, right of=etrans, xshift=0.5cm] {Forward\\interpretability};
  \node (reverse) [criteria, right of=retro, xshift=5.2cm] {Reverse\\interpretability};
  
  \draw[->]   (world)   to[bend left=14]  node {compression via\\bisimulation} (reduced.west);
  \draw[->]   (world)      -- node {phase space\\expansion} (predictive);
  \draw[->]   (world)   to[bend right=14] node {time\\reversal} (retro.west);
  \draw[->]   (reduced)    -- node {compression via\\ quasi-probabilities} (gentrans);
  \draw[->]   (predictive) -- node {compression via\\bisimulation} (etrans);
  \draw[->]   (retro) -- node {phase space\\expansion} (rev_beliefs);
  
  \end{tikzpicture}
  \caption{\centering 
  \small{Recommendations for building world optimal models, including implementations (boxes), transformations (arrows), and design criteria (ellipses).}}
  \label{fig:diagram}
\end{figure}

\section{Generating interfaces via transducers}
\label{sec:formal_foundations}

This section formalises the notions of `world model' and `interface'.
In the following, uppercase letters (e.g. $X,Y$) are used to denote random variables and lowercase (e.g. $x,y$) their realisations, calligraphic letters (e.g. $\mathcal{X},\mathcal{Y}$) denote the sets over which they take values, and the symbol $\Delta$ (as in $\Delta(\mathcal{X}),\Delta(\mathcal{Y})$) is used to denote the collection of all distributions over those sets. 
We use the shorthand notation $p(x|y) = \Pr(X=x|Y=y)$ to express probabilities when there is no risk of ambiguity, and assume that equalities of the form $p(x|y,z) = p(x|y)$ hold for all realisations that can take place with non-zero probability. 
$\mathbb{N}=\{0,1,2,\ldots\}$ corresponds to zero-based numbering, and we use the following abbreviations: $\bm x_{a:b}=(x_a,\dots,x_b)$, $\bm x_{:b} = \bm x_{0:b}$, $\bm x_{a:} = \bm x_{a:\infty}$, and $\bm x_: = \bm x_{0:\infty}$.

\subsection{World models}

We operationalise interfaces as descriptions of how actions turn into outcomes for a particular agent.
\begin{definition}
\label{def:interface}
    An \textbf{interface} $\mathcal{I}(\bm Y|\bm A)$ is a collection of distributions $\{p(\bm y_{:t}|\bm a_{:}), t\in\mathbb{N}\}$ corresponding to a stochastic process over outcome sequences $\bm y_: 
    \in\mathcal{Y}^{\mathbb{N}}$ conditioned on action sequences $\bm a_:
    \in\mathcal{A}^{\mathbb{N}}$.
    An interface is \textbf{anticipation-free} if $p(\bm y_{:t}|\bm a_{:}) = p(\bm y_{:t}|\bm a_{:t})$ for all $t\in\mathbb{N}$.
\end{definition}
Essentially, an interface describes a semi-infinite stochastic process \citep{kallenberg1997foundations, loomis2023topology} for each action sequence $\bm a_{:}$ while being agnostic to the agent's computational capabilities, architecture, or internal functioning. 
Interfaces can be constructed from an underlying world model that specifies how actions turn into outcomes. 
Next, we introduce a general notion of a world model in terms of statistical sufficiency (\autoref{app:sufficient_stats}), and use $h_t = (a_t,y_t)$ so that $\bm h_{:t}$ denotes the joint history of the interface up to time $t$.

\begin{definition}
\label{def:world_model}
    A \textbf{world model} for an interface $\mathcal{I}(\bm Y|\bm A)$ is a collection of distributions $\{p(\bm s_{:t}|\bm h_{:}),t\in\mathbb{N}\}$ corresponding to a 
    stochastic process over state sequences $\bm s_: \coloneqq (s_0,s_1,\dots)\in\mathcal{S}^{\mathbb{N}}$ that satisfies
    \begin{equation}
    \label{eq:cond_world_model}
    (1)~p(y_{t}|\bm s_{:t},\bm a_{:t},\bm y_{:t-1})
    =p(y_{t}|s_{t},a_{t})
    \quad\text{and}\quad
    (2)~p(\bm s_{:t}| \bm h_{:t-1},\bm a_{t:})
    = p(\bm s_{:t}| \bm h_{:t-1})
    \quad\forall t\in\mathbb{N}.
    \end{equation}
\end{definition}
 
Thus, world models are candidate mechanisms for implementing the statistical relationships between actions and outcomes, taking the form of auxiliary processes $S_t$ that encapsulate relevant information between past events and present outcomes (condition~1) while guaranteeing time's arrow so future actions cannot affect previous world states or outcomes (condition~2). 
We may informally denote a world model simply by $S_t$ when it is unambiguous from context. 
This definition includes models of the `external world' such as \emph{partially observed Markov decision processes} (POMDPs), as well as their `epistemic' counterparts, \emph{belief MDPs}~\citep{kaelbling1998planning}, as explained in~\autoref{sec:transducer_types}. 
The unifying property of all world models is presented next (proof in~\autoref{app:proof_of_factorisation_forward}).

\begin{lemma}
\label{lemma:world_model}
A process $S_t$ is a world model for an anticipation-free interface $\mathcal{I}(\bm Y|\bm A)$ if and only if 
\vspace{-3pt}
\begin{align}
\label{eq:world_model}
    p(\bm y_{:t},\bm s_{:t+1}|\bm a_{:})
    = 
    p(s_0) \prod_{\tau=0}^{t}
    p(y_\tau|s_{\tau},a_{\tau}) p(s_{\tau+1}|\bm s_{:\tau},\bm h_{:\tau}),
    \quad
    \forall t\in\mathbb{N}.
    \end{align}
\end{lemma}
\vspace{-1pt}
Thus, world models let us express interfaces in terms of probabilistic graphical models~\citep{koller2009probabilistic}. 
Among other things, \autoref{eq:world_model} can be used to generate outcomes for given sequences $\bm a_{:t}$ and $\bm s_{:t}$ by sampling 
$p(\bm y_{:t}|\bm s_{:t},\bm a_{:t}) = \prod_{\tau=0}^t p(y_\tau|s_\tau,a_\tau)$. 
In this sense, we say that the world model $S_t$ \textbf{\textit{generates}} the interface $\mathcal{I}(\bm Y|\bm A)$, and that the graphical model outlined in \autoref{eq:world_model} establishes a \textbf{\textit{presentation}} of the interface. These ideas are illustrated by an example in~\autoref{app:example}.

\subsection{Transducers}
\label{sec:transducers}

Sampling sequences of world model's states can be highly non-trivial if their dynamics are non-Markovian. 
This issue can be avoided by restricting ourselves to building world models using \emph{transducers}~\citep{barnett2015computational}, a computational structure that is introduced next.

\begin{definition}
    \label{def:transducer}
    A \textbf{transducer} is a tuple $\big(\mathcal{S}, \mathcal{A}, \mathcal{Y}, \kappa, \dist{p} \big)$, where $\mathcal{S}$ is a set of memory states, $\mathcal{A}$ and $\mathcal{Y}$ are sets of inputs and outputs, $\kappa:\mathcal{A}\times\mathcal{S}\times\mathbb{N}\to\Delta(\mathcal{Y}\times\mathcal{S})$ is a Markov kernel of the form 
    $\{\kappa_\tau(y,\tilde{s}|a,s): s,\tilde{s} \in \mathcal{S}, y \in \mathcal{Y},a \in \mathcal{A}, \tau\in \mathbb{N}\}$, 
    and $\dist{p}\in\Delta(\mathcal{S})$ is an initial distribution over states. 
\end{definition}

If a transducer's memory can only take $|\mathcal{S}|=n$ different states, then their transitions can be described via substochastic matrices $T_\tau^{(y|a)}$ of the form
\begin{align}
\label{eq:transducer_matrices}
        T_\tau^{(y|a)} \coloneq \sum_{i=1}^n\sum_{j=1}^n 
        \kappa_\tau(y,s_i|a,s_j) \bm e_{i} \bm e_j^\intercal~,
\end{align}
where $\bm e_k$ is a binary vector with a 1 at the $k$-th position and zeros elsewhere. 
These transducers are also known as stochastic automata~\citep{claus2013stochastische,cakir2021theory}, generalising deterministic automata~\citep{minsky1967computation} by using stochastic transitions to generate outputs and update their state. 
Running a transducer generates an interface $\mathcal{I}(\bm Y|\bm A)$ given by its inputs and outputs according to 
\begin{equation}
    p(\bm y_{:t}\bm s_{:t+1}|\bm a_{:}) 
    = p(s_0) \prod_{\tau=0}^{t}\kappa_\tau(y_\tau, s_{\tau+1}|a_\tau,s_{t}),
    \label{eq:transducer}
\end{equation}
providing a graphical model that can be used to simulate the interface (see \autoref{fig:some_fig}).
Comparing \autoref{eq:transducer} with \autoref{lemma:world_model} let us to formalise this fact as follows.

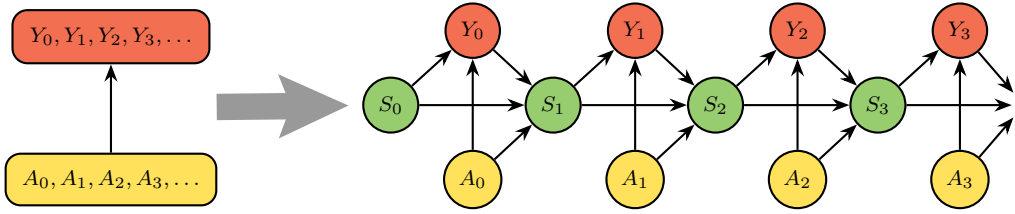
\begin{figure}[ht]
\centering
\begin{tikzpicture}[
    every node/.style={font=\footnotesize, minimum width=0.5em,thick, minimum size=0.3cm}, 
    state/.style={circle, draw, fill=YellowGreen},
    output/.style={circle, draw, fill=Red!70},
    input/.style={circle, draw, fill=Goldenrod!90},
    node distance=1.4cm,
    >=Stealth
    ]

\node[state] (S0) {$S_0$};
\node[state, right=of S0] (S1) {$S_1$};
\node[state, right=of S1] (S2) {$S_2$};
\node[state, right=of S2] (S3) {$S_3$};
\node[right=of S3] (S4) {};

\node[input,below right of=S0, xshift=0.1cm] (A0) {$A_0$};
\node[input,below right of=S1, xshift=0.1cm] (A1) {$A_1$};
\node[input,below right of=S2, xshift=0.1cm] (A2) {$A_2$};
\node[input,below right of=S3, xshift=0.1cm] (A3) {$A_3$};

\node[output,above right of=S0, xshift=0.1cm] (Y0) {$Y_0$};
\node[output,above right of=S1, xshift=0.1cm] (Y1) {$Y_1$};
\node[output,above right of=S2, xshift=0.1cm] (Y2) {$Y_2$};
\node[output,above right of=S3, xshift=0.1cm] (Y3) {$Y_3$};

\node[rectangle, draw, fill=Goldenrod!90, inner sep=0.22cm, rounded corners=0.2cm, left=3cm of A0] (Aseq) {$A_0, A_1, A_2, A_3, \ldots$};
\node[rectangle, draw, fill=Red!70, inner sep=0.23cm, rounded corners=0.2cm, above of=Aseq, yshift=0.5cm] (Yseq) {$Y_0, Y_1, Y_2, Y_3, \ldots$};

\draw[->,line width=0.28mm] (Aseq) -- (Yseq);

\draw[->,line width=0.28mm] (S0) -- (Y0);
\draw[->,line width=0.28mm] (A0) -- (Y0);
\draw[->,line width=0.28mm] (A0) -- (S1);
\draw[->,line width=0.28mm] (S0) -- (S1);
\draw[->,line width=0.28mm] (Y0) -- (S1);

\draw[->,line width=0.28mm] (S1) -- (Y1);
\draw[->,line width=0.28mm] (A1) -- (Y1);
\draw[->,line width=0.28mm] (A1) -- (S2);
\draw[->,line width=0.28mm] (S1) -- (S2);
\draw[->,line width=0.28mm] (Y1) -- (S2);

\draw[->,line width=0.28mm] (S2) -- (Y2);
\draw[->,line width=0.28mm] (A2) -- (Y2);
\draw[->,line width=0.28mm] (A2) -- (S3);
\draw[->,line width=0.28mm] (S2) -- (S3);
\draw[->,line width=0.28mm] (Y2) -- (S3);

\draw[->,line width=0.28mm] (S3) -- (Y3);
\draw[->,line width=0.28mm] (A3) -- (Y3);
\draw[->,line width=0.28mm] (A3) -- (S4);
\draw[->,line width=0.28mm] (S3) -- (S4);
\draw[->,line width=0.28mm] (Y3) -- (S4);
      
\draw[-{Stealth[width=23pt,length=26pt,inset=4pt]}, line width=10pt, color=gray!80, shorten <= 40pt, shorten >= 5pt] 
      ($(Aseq)!0.5!(Yseq)$) -- (S0);

\end{tikzpicture}
\caption{\centering 
\small{Illustration of an interface (left) and a possible unravelling of it via a presentation with a world model built from the memory states of a transducer (right), as given by~\autoref{eq:transducer}.}}
\label{fig:some_fig} 
\end{figure}

\begin{lemma}
\label{lemma:transducer_as_world_model}
     Transducers correspond to world models of anticipation-free interfaces whose dynamics satisfy the Markov condition
     $p(s_{\tau+1}|\bm s_{:\tau},\bm h_{:\tau})=
     p(s_{\tau+1}| s_{\tau},h_{\tau})$ for all $\tau\in\mathbb{N}$.
\end{lemma}
We may denote a transducer informally as $(S_t, A_t, Y_t)$ when it is unambiguous from the context, and describe its memory state $S_t$ as a world model when appropriate. 
Complementary characterisations of transducers in terms of sufficient statistics and information properties are provided in~\autoref{app:proof_transducer_faces}.

Transducers can be seen as reflecting the memory structure of interfaces. 
In particular, an interface $\mathcal{I}(\bm Y|\bm A)$ is said to be \textbf{\textit{memoryless}} if $S_t=0$ is a valid transducer, and is \textbf{\textit{fully} observable} if $S_t = Y_t$ yields a valid transducer --- including Markov decision processes (MDPs) as a main example. 
We next characterise the statistics of such interfaces, which clarifies that elaborate world models are required by interfaces with non-Markovian dynamics (proof in \autoref{app:interface_memory}).

\begin{lemma}
\label{lemma:interface_memory}
An interface is fully observable if and only if $p(y_{t+1}|\bm y_{:t},\bm a_:) = p(y_{t+1}| y_t,a_t)$, and is memoryless if and only if $p(\bm y_{:t}|\bm a_:) = \prod_{\tau=0}^t p(y_\tau|a_\tau)$. 
\end{lemma}

Note that if $p(\bm y_{:t}|\bm a_:) = p(\bm y_{:t})$, corresponding to `contemplative' or passive agents that do not act but only sense, then transducers reduce to hidden Markov models~\citep{ephraim2002hidden}.

\subsection{General classes of transducers}
\label{sec:transducer_types}
Transducers use their kernels $\kappa_\tau$ to generate $(s_{t+1},y_t)$ jointly from $(s_t,a_t)$, corresponding to what has been  described as a `Mealy' machines~\citep{virgo2023unifilar,bonchi2024effectful}. 
Simpler computational structures can be obtained by imposing constraints in the kernel as follows: 
\begin{itemize}
    \item \textbf{\textit{Input-Moore transducers}} generate outputs ignoring the current input, corresponding to kernels of the form
        $\kappa_\tau(y,\tilde{s}| a, s) = \mu_\tau(y | s) \nu_\tau(\tilde{s}| y,a, s)$.
    \item \textbf{\textit{Output-Moore transducers}}
    update their state without considering the current output, corresponding to kernels of the form
        $\kappa(y,\tilde{s}| a, s) = \mu_\tau(y | a, s) \nu_\tau(\tilde{s}| a, s)$. 
    \item 
    \textbf{\textit{I-O Moore transducers}} satisfy both previous conditions, corresponding to kernels of the form
        $\kappa(y,\tilde{s}| a, s) = \mu_\tau(y | s) \nu_\tau(\tilde{s}| a, s)$.
\end{itemize}
Based on ideas from automata theory~\citep{lee2017introduction}, 
input-Moore transducers serve as models for interfaces that satisfy the stronger anticipation-free condition $p(\bm y_{:t}|\bm a_{:}) = p(\bm y_{:t}|\bm a_{:t-1})$, corresponding to scenarios where $Y_t$ takes place before $A_t$, thus reflecting a time-indexing convention. 
In contrast, output-Moore transducers build on the hidden Markov models literature~\citep{riechers2016exact}, and are used to represent (non-quantum) physical processes whose evolution is not affected by the observations made by the agent --- as opposed to epistemic processes such as belief MPDs, in which the opposite happens (see~\autoref{sec:epistemic_worlds}). 
POMDPs can be shown to be examples of either output-Moore or I-O Moore transducers depending on the specific definition, as explained in~\autoref{app:POMDPs}.

\section{Minimal world models}
\label{sec:reduction}

After setting the formal foundations of world models, and establishing transducers as a natural way to construct them, we now investigate how to build \emph{minimal} world models.

\subsection{Reducing world models}
\label{sec:minimality}

We begin by showing that all interfaces have at least one transducer presentation, and hence one can focus on transducers without loss of generality (see the proof in \autoref{app:proof_funes}).

\begin{lemma}\label{lemma:funes}
The world model $S_t = \bm H_{:t-1}$ yields a transducer that generates the interface $\mathcal{I}(\bm Y|\bm A)$.
\end{lemma}

This world model is far from parsimonious, resembling Borges' character \textit{Funes the memorious} in its inability to forget. 
This suggests the importance of `reducing' world models. 
To formalise this, we extend the notion of MDP homomorphism~\citep{ravindran2003smdp} to transducers as follows.

\begin{definition}
\label{def:minimal_world}
A \textbf{homomorphism} between transducers $(S_t, A_t, Y_t)$ and $(S'_t, A'_t, Y'_t)$ is given by the mappings $\langle \phi:\mathcal{S}\to\mathcal{S}', f:\mathcal{Y}\to\mathcal{Y}', g:\mathcal{A}\to\mathcal{A}' \rangle$ satisfying two compatibility conditions:
\begin{enumerate}
    \item[(i)] $\Pr\!\big(Y'_t=y'| S'_t=\phi(s),A'_t=g(a)\big) = \sum_{y\in f^{-1}(y')}\Pr\!\big(Y_t=y| S_t=s,A_t=a\big) $.
    \item[(ii)] $\Pr\!\big(S'_{t+1}=s'| S'_t=\phi(s),H'_t=(f(y),g(a))\big) 
    \!=\! 
    \sum_{s\in \phi^{-1}(s')}\Pr\!\big(S_{t+1}\!=s| S_t=\tilde{s},H_t=(y,a)\big)$ and 
    $\Pr\!\big(S'_{0}=s'\big) 
    = \sum_{s\in \phi^{-1}(s')}\Pr\!\big(S_{0}=s\big)$.
\end{enumerate}
    A \textbf{reduction} of a world model $S_t$ into a world model $S'_t$ is a homomorphism $\langle \phi, f, g \rangle$ between the transducers $(S_t, A_t, Y_t)$ and $(S'_t, A_t, Y_t)$, where $f:\mathcal{Y}\to\mathcal{Y}$ and $g:\mathcal{A}\to\mathcal{A}$ are identity mappings and $\phi$ is a surjective map $\phi: S_t \to S'_t$.
    Two transducers are \textbf{isomorphic} if they are reductions of each other, and a transducer is \textbf{minimal} if all its reductions are isomorphic to itself.
\end{definition}

A world reduction can be informally described as a coarse-graining $\phi$ between the memory states of two transducers of the same interface. 
Condition (i) above ensures that outcomes are generated with the same statistics, and (ii) that the resulting world model is still Markovian --- as can be confirmed by relating it with the notion of `lumpability' of Markov chains~\citep{tian2006lumpability}.  These properties guarantee that reductions do not distort the corresponding interface (proof in~\autoref{app:proof_homo_generate}).

\begin{lemma}
    \label{lemma:homo_generate}
    A world model reduction yields a transducer presentation of the original transducer.
\end{lemma}

The next two sections study different approaches to look for minimal world models.\footnote{Minimality can also be studied via the entropy of the world's dynamics. Interestingly, minimal entropy models may not coincide with the models with fewer states --- although the two coincide for predictive models~\citep{loomis2019strong}.}

\subsection{Reduction via bisimulation}
\label{sec:bisimulation}

A natural way to reduce a world model is via the notion of bisimulation, which is a way of merging states that behave equivalently~\citep{givan2003equivalence}. 
Here we leverage previous work on bisimulations for hidden Markov models \citep{jansen2012belief} to define bisimulations of transducers. 

\begin{definition}
\label{def:bisimulation}
For a given transducer with world model $S_t$ and kernel $\kappa_t$, a \textbf{bisimulation} is an equivalence relation $\mathcal{B}_t \subseteq \mathcal{S} \times \mathcal{S}$ such that $s,s'\in\mathcal{S}$ are equivalent if they satisfy two conditions:
\begin{enumerate}
    \item[(i)] $p_t(y|s,a) = p_t(y|s',a)$, where $p_t(y|s,a) = \sum_{s''\in\mathcal{S}} \kappa_t(y,s''|s,a)$.
    \item[(ii)] $p_t(C|s,a) = p_t(C|s',a)$ for all equivalence classes $C\subseteq\mathcal{S}$, where $p_t(C|s,a) = \sum_{y\in\mathcal{Y}}\sum_{s''\in C}\kappa_t(y,s''|s,a)$.
\end{enumerate}
\end{definition}

There is a direct correspondence between world model reductions (\autoref{def:minimal_world}) and bisimulations, as shown next by adapting \cite[Theorem 3]{taylor2008bounding} to our setup (proof in \autoref{app:proof_lemma_bisim_homo}).

\begin{proposition}
\label{lemma:bisim_homo}
    $\phi: S_t \to S'_t$ is a reduction of world models if and only if the equivalence relation it induces, with equivalence classes given by $\phi^{-1}(s')=\{s\in\mathcal{S}:\phi(s)=s'\}$, is a bisimulation.
\end{proposition}

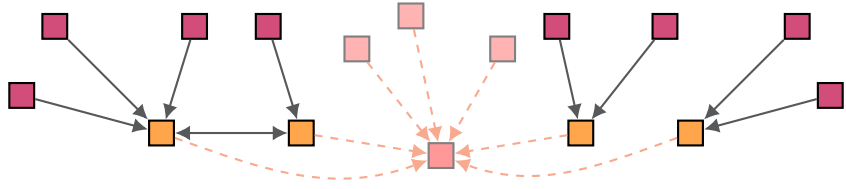
\begin{figure}
\centering
\begin{tikzpicture}[
    node distance=1.5cm,
    every node/.style={minimum size=0.33cm, thick,
        square/.style={rectangle, draw=black, fill=purple!70},
        sink/.style={rectangle, draw=black, fill=orange!70},
        sink/.style={rectangle, draw=black, fill=orange!70},
        square_neg/.style={rectangle, draw=black!50, fill=red!30},
        sink_neg/.style={rectangle, draw=black!50, fill=red!40}},
        normal_arrow/.style={draw=Black!80,thick},
        arrow_red/.style={draw=Red!40, dashed, thick},
]
    \node[sink] (center1) {};
    \node[sink, right=of center1] (center2) {};
    \node[square, left=of center1, yshift=0.5cm] (sq1) {};
    \node[square, above left=of center1] (sq2) {};
    \node[square, above left=of center2] (sq3) {};
    \node[square, above right=of center1] (sq4) {};
    
    \draw[<->,normal_arrow] (center1) to (center2);
    \draw[->,normal_arrow] (sq1) to (center1);
    \draw[->,normal_arrow] (sq2) to (center1);
    \draw[->,normal_arrow] (sq3) to (center1);
    \draw[->,normal_arrow] (sq4) to (center2);

    \node[sink_neg, right=of center2, yshift=-0.3cm] (center_neg_1) {};
    \node[square_neg, above right=of center_neg_1, xshift=-0.6cm] (cc1) {};
    \node[square_neg, above=of center_neg_1, xshift=-0.6cm, xshift=0.2cm] (cc2) {};
    \node[square_neg, above left=of center_neg_1, xshift=0.3cm] (cc3) {};
    
    \draw[->,arrow_red] (cc1) to (center_neg_1);
    \draw[->,arrow_red] (cc2) to (center_neg_1);
    \draw[->,arrow_red] (cc3) to (center_neg_1);
    \draw[->,arrow_red] (center1) to [out=-20,in=-160] (center_neg_1);
    \draw[->,arrow_red] (center2) to (center_neg_1);
    
    \node[sink, right=of center_neg_1, yshift=0.3cm] (center3) {};
    \node[sink, right=of center3, xshift=-0.4cm] (center4) {};
    
    \node[square, above left=of center3,xshift=1.1cm] (sq6) {};
    \node[square, above right=of center3,xshift=-0.3cm] (sq7) {};

    \draw[->,normal_arrow] (sq6) to (center3);
    \draw[->,normal_arrow] (sq7) to (center3);

    \node[square, above right=of center4] (sq8) {};
    \node[square, right=of center4, yshift=0.5cm] (sq9) {};
    
    \draw[->,normal_arrow] (sq8) to (center4);
    \draw[->,normal_arrow] (sq9) to (center4);

    \draw[->,arrow_red] (center4) to [out=-160,in=-20] (center_neg_1);
    \draw[->,arrow_red] (center3) to (center_neg_1);
    
\end{tikzpicture}
\caption{\centering \small{Illustration of the minimisation of world models. Purple boxes represent reducible models and orange boxes represent minimal ones, and arrows correspond to reductions. Red boxes are generalised models following quasi-probabilities, which (if allowed) establish global minima.}}
\end{figure}

Together with \autoref{lemma:homo_generate}, this result confirms that the bisimulation of a transducer yields another transducer presentation for the same interface. 
This has a simple and yet powerful implication: a full reduction of a given transducer can be attained by coarse-graining all bisimilar states.

There may be cases where bisimulations do not produce the most efficient world model that generates a given interface, since reducing a particular transducer usually does not lead to a global minimum.  
To investigate this claim, we consider a world model with $|\mathcal{S}|=n$ states and build vectors $\bm w(\bm h_{:t})\in\mathbb{R}^n$ containing the probabilities of generating $\bm y_{:t}$ given $\bm a_{:t}$ when starting from different world states, so that its $k$-th coordinate is 
$[\bm w(\bm h_{:t})]_k = \Pr(\bm Y_{:t} = \bm y_{:t} | \bm A_{:t}=\bm a_{:t}, S_0 = s_k)$. 
Intuitively, if vectors $\bm w(\bm h_{:t})$ for different $\bm h_{:t}$ are linearly dependent, some of their dimensions (and, hence, the corresponding world states) are still in some sense redundant. 
Crucially, the coarse-grainings associated to bisimulation can only lump states that have identical components, but cannot reduce linear dependencies between states more generally. Relaxing the criteria for merging states (e.g. via bisimulation metrics~\citep{ferns2014bisimulation}) does not solve this issue, as this introduces distortion in the interface due to the imprecisions allowed in the state merging procedure.

These ideas can be made concrete by studying the \emph{canonical dimension} of a transducer $\mathcal{T}$, given by\footnote{If a transducer has $|\mathcal{S}|=n$ memory states, then $\lim_{m\to \infty} \text{dim}(U_m) = \text{dim}(U_{n-1})$~\cite[Prop.~4.3]{cakir2021theory}.}
    \begin{equation}\label{eq:cannonical_dimension}
        d(\mathcal{T}) \coloneqq \lim_{m\to \infty} \text{dim}(U_m), 
        \qquad
        \text{where}
        \quad 
        U_m = \text{Span} \{\bm w(\bm h_{:t}):  t\leq m\}\subseteq \mathbb{R}^n. 
    \end{equation}
The canonical dimension is an important indicator of the compressibility of a transducer as shown next, whose proof can be found in~\citep[Cor.~4.9]{cakir2021theory} --- also see~\citep{ito1992identifiability,balasubramanian1993equivalence} for related results in hidden Markov models. 

\begin{theorem}\label{theo:hmm_mimimal}
    If $\mathcal{T}$ is a transducer with $|\mathcal{S}|=n\in\mathbb{N}$, then $d(\mathcal{T})=n$ implies that there are no transducers with fewer memory states that can generate the same interface.
\end{theorem}

The minimal bisimulation of a transducer $\hat{\mathcal{T}}$ with world states in $\hat{\mathcal{S}}$ could still exhibit $d(\hat{\mathcal{T}}) < |\hat{\mathcal{S}}|$. 
In fact, there are interfaces for which no transducer reaches $d(\mathcal{T}) = |\mathcal{S}|$. 
Even if there exists a transducer with $d(\hat{\mathcal{T}}) = |\mathcal{S}|$, we are not aware of any general algorithm that can directly build it.\footnote{In fact, the relatively simpler case of reducing hidden Markov models is still not fully solved~\citep{vidyasagar2011complete}, although algorithms that can address some cases have been developed~\citep{huang2015minimal, ohta2021realization}.}

\subsection{Reduction via pseudo-probabilities and generalised transducers}
\label{sec:ghmm}

This section focuses on the reduction of world models with a finite number of states $|\mathcal{S}|=n$ but $d(\mathcal{T}) < n$. As discussed in \autoref{sec:transducers}, 
the probabilities of $\bm y_{:t}$ given $\bm a_{:t}$ can then be calculated as
\begin{equation}\label{eq:matrix_notation}
    p(\bm y_{:t}|\bm a_{:t}) = \bm 1^\intercal \cdot \Big( \prod_{\tau=0}^t T^{(y_\tau|a_\tau)}_\tau \Big) \cdot \bm p,
\end{equation}
where $\bm 1^\intercal$ is a transposed vector with $n$ ones as components. 
Normally, the substochastic matrices $T^{(y|a)}_t$ and the initial distribution $\bm p$ are assumed to contain only non-negative terms. However, a more general class of transducers can be explored by removing this constraint and considering \emph{quasi-distributions} $\bm v\in\mathbb{R}^n$, which may have negative components but still satisfy $\sum_{i=1}^n v_i=1$, and quasi-stochastic matrices whose columns are quasi-distributions~\citep{balasubramanian1993equivalence,upper1997theory}. This leads to the following generalisation of a transducer.

\begin{definition}
    A \textbf{generalised transducer} for an interface $\mathcal{I}(\bm Y|\bm A)$ is a tuple $(\mathcal{S}, \mathcal{A}, \mathcal{Y}, \{A^{(y|a)}\}, \bm v, \bm u)$ with $\bm u,\bm v\in\mathbb{R}^n$
    and $A^{(y|a)}\in \mathbb{R}^{n\times n}$ that satisfy
    \begin{align}
        p(\bm y_{:t}|\bm a_{:t}) 
        = \bm u^\intercal \cdot \Big( \prod_{\tau=0}^t A^{(y_\tau|a_\tau)}_\tau \Big) \cdot \bm v
        \qquad
        \forall\bm y_{:t}\in\mathcal{Y}^{t+1},\bm a_{:t}\in\mathcal{A}^{t+1}.
    \end{align}
\end{definition}
Generalised transducers are useful because, in contrast to standard transducers (or POMDPs), they can always be reduced to find representations with a minimal number of states, as shown next.

\begin{theorem}\label{theo:gmmm_reduction}
    A generalised transducer $\tilde{T}$ with $d(\tilde{T})<n$ can always be reduced via a linear transformation into another transducer that generates the same interface using fewer states.
\end{theorem}

This result follows directly from the proofs provided in~\cite[Ch.~3]{balasubramanian1993equivalence} and related results can be found in~\citep{upper1997theory,vidyasagar2011complete}, in which reductions correspond to linear projections. Notably, these proofs lead to practical algorithms that can be used to efficiently reduce transducers with $d(\tilde{T})<n$ (see \autoref{app:algorithms_gtransducers}). 
In this way, generalised transducers achieve a minimal computational complexity at the cost of introducing an opaque world model whose trajectories cannot be sampled (due to the quasi-probabilities), which results in a substantial lack of interpretability.

\section{Forward interpretability via epistemic world models}
\label{sec:epistemic_worlds}

The previous section shows how computational efficiency can be achieved by either compressing memory state spaces with bisimulations or by allowing memory states of transducers to be encoded by quasi-probabilities.
While the latter generally yields higher efficiency, this comes at the cost of making those reduced world models highly uninterpretable due to the possible presence of negative probabilities. 
This section takes a different route by investigating specific types of world models that focus on interpretability, bringing insights about what agents can learn.

\subsection{Beliefs as world models}
\label{sec:worlds_of_beliefs}

Let us start by reviewing properties of certain classes of world models
that make them learnable by agents in real time. 
A world model $S_t$ is \textbf{\textit{predictive}} if it contains only past information, which in information-theoretic terms corresponds to $I(S_{t};\bm Y_{t:}| \bm H_{:t-1},\bm A_{t:}) = 0$. 
A world model is \textbf{\textit{observable}} if $S_{t+1} = f_t(\bm H_{:t})$, i.e. if it can be estimated from action-output history via mappings $f_t$. 
Finally, a world model is \textbf{\textit{unifilar}} if $S_{t+1}=\hat{f}_t(S_t, A_t, Y_t)$, so its state can be deterministically updated given inputs and outputs. 
Thus, observable models are always predictive, and unifilar models are observable if there is no randomness in the world's initial condition. Moreover, unifilar models correspond to transducer whose kernels have the form $\kappa_\tau(y,\tilde{s}|a,s) = \delta^{\tilde{s}}_{\hat{f}_\tau(y,a,s)}\mu_\tau(y|a,s)$.

The literature contains several procedures for building observable world models from non-observable ones (see~\citep{subramanian2022approximate, ni2024bridging} for general reviews and~\citep{virgo2021interpreting, biehl2022interpreting, virgo2023unifilar} for a categorical formulation). 
These approaches suggest to expand the phase space of world models 
from elements in $\mathcal{S}$ to distributions over those $\Delta(\mathcal{S})$, henceforth called \emph{belief states}. 
This idea has been extensively studied for POMDPs via the notion of \emph{belief MDP}~\citep{kaelbling1998planning}. 
We extend these ideas to more general transducers. 
\begin{definition}
    A \textbf{belief transducer} over a transducer $(\mathcal{S},\mathcal{A},\mathcal{Y},\kappa,\dist{p})$ 
    is another transducer $\big(\mathcal{B},\mathcal{A},\mathcal{Y},\hat{\kappa}, \delta_{\dist{b}_0} \big)$ where 
    $\mathcal{B}\subseteq\Delta(\mathcal{S})$ is a set of belief states, $\hat{\kappa}:\mathcal{A}\times\mathcal{B}\times\mathbb{N}\to\Delta(\mathcal{Y}\times\mathcal{B})$ is a Markov kernel of the form 
    $\{\hat{\kappa}_\tau(y,b'|a,b): b,b'\in\mathcal{B},a\in\mathcal{A},y\in\mathcal{Y},\tau\in\mathbb{N}\}$ 
    such that $\hat{\kappa} = F\{\kappa\}$ for some functional $F$, 
    and $\dist{b}_0\in\mathcal{B}$ is an initial belief. A belief transducer is said to be \textbf{faithful} if it generates the same interface as the original transducer.
\end{definition}
A natural way to define beliefs about an underlying world model $S_t$ is via \textbf{\textit{predictive Bayesian beliefs}} corresponding to the posterior distribution $\dist{b}_t(s_t) \coloneq \Pr(S_t=s_t | \bm H_{:t-1} = \bm h_{:t-1})$. 
The dynamics of the updates of such beliefs are described by Bayesian prediction~\citep{jazwinski1970stochastic, sarkka2023bayesian}, and their properties have been further studied under the name of `mixed-states' in computational mechanics~\citep{riechers2018spectral, jurgens2021shannon}. 
Building on this literature, we show that the predictive Bayesian beliefs of the memory states of transducers are unifilar and can be used to generate the same interface (proof in \autoref{app:beliefs_generate_interfaces}).

\begin{proposition}
\label{prop:predictive_beliefs}
    If $(S_t, A_t, Y_t)$ is a transducer and  $\dist{B}_t$ is the predictive Bayesian belief of $S_t$, then $(\dist{B}_t, A_t, Y_t)$ is a unifilar belief transducer whose state dynamics are given by
    \begin{equation}\label{eq:update_mixed_states}
        \dist{b}_{t+1}(s_{t+1}) 
        =  
        \frac{1}{Z}\sum_{s_t} p(y_t, s_{t+1}|a_t,s_t) \dist{b}_t(s_t),
    \end{equation}
    with $Z$ a normalisation constant. 
    Moreover, $(\dist{B}_t, A_t, Y_t)$ is faithful if $\dist{b}_0=p(s_0)$.
\end{proposition}

Interestingly, I-O Moore transducers (\autoref{sec:transducer_types}) also allow for \textbf{\textit{postdictive Bayesian beliefs}}\footnote{In general, \textbf{\textit{postdictive world models}} $S_t$ satisfy $I(S_{t};\bm Y_{t+1:}| \bm H_{:t},\bm A_{t+1:}) = 0$.} of the form $\dist{d}_t(s_t) \coloneq \Pr(S_t=s_t | \bm G_{:t} = \bm g_{:t})$ with $g_t=(y_t,a_{t-1})$, in which $Y_t$ is used to infer $S_t$. 
Our next result explains how predictive and postdictive Bayesian beliefs and MSPs relate, and how they set the bases for Bayesian and Kalman filtering  (proof in~\autoref{app:postdictive_stuff}).

\begin{proposition}
    \label{prop:postdictive_stuff}
    If $(S_t, A_t, Y_t)$ is a I-O Moore transducer and  $\dist{D}_t$ is the postdictive Bayesian belief of $S_t$, then $(\dist{D}_t, A_t, Y_t)$ is a belief transducer whose state dynamics are given by
    \begin{align}\label{eq:postdictive_belief_update}
    \dist{d}_{t+1}(s_{t+1})
    =
    \frac{p(y_{t+1}|s_{t+1})}{Z'} \sum_{s_{t}} p(s_{t+1}| s_{t},a_{t}) \dist{d}_{t}(s_{t}),
    \end{align}
    with $Z'$ a normalisation constant. 
    Moreover, $(\dist{D}_t, A_t, Y_t)$
    is faithful if $\dist{d}_0=p(s_0)$.
\end{proposition}

\begin{corollary}\label{cor:pred_up}
    Their predictive and postdictive Bayesian beliefs of I-O Moore transducers can be updated as $\dist{b}_{t-1}\xrightarrow{\text{predict}} \dist{d}_{t} \xrightarrow{\text{update}} \dist{b}_{t}$ following the `predict-update' process from Bayesian filtering.
\end{corollary}

A faithful belief transducers can be said to provide a 
\textbf{\textit{mixed-state presentation}} (MSP) of the underlying transducer, extending previous work on MSPs of hidden Markov models~\citep{jurgens2021divergent,jurgens2021shannon}.
MSPs are generally not minimal, as they tend to have different mixed-states that are bisimilar. The reduction of these is studied in the next section.

\subsection{Minimal predictive world models}
\label{sec:minimal_predictive_models}

Following \cite{barnett2015computational}, we now present a method to build an observable world model directly from an interface $\mathcal{I}(\bm Y|\bm A)$ without the need to bootstrap from another world model. 
For this, consider an equivalence relation between histories in which $\bm h_{:t-1}\sim_\epsilon \bm h'_{:t-1}$ when 
    \begin{equation}
    \label{eq:epsilon}
        p(\bm y_{t:t+T}|\bm h_{:t-1}, \bm a_{t:t+T}) 
        = p(\bm y_{t:t+T}|\bm h'_{:t-1}, \bm a_{t:t+T}),
        \quad \forall \bm y_{t:t+L}, \bm a_{t:t+L}, L\in\mathbb{N}.
    \end{equation}
Let's denote by $\epsilon_t$ the coarse-graining mapping that assigns each history to its corresponding equivalence class $\epsilon_t(\bm h_{:t-1}) = [\bm h_{:t-1}]_{\sim_\epsilon}$, and define $M_t = \epsilon_t(\bm H_{:t-1})$. 
This construction is known as \emph{predictive state representations}~\citep{littman2001predictive,singh2004psr} and \emph{instrumental states}~\citep{kosoy2019imperceptible}, and is based on older ideas for stochastic processes (without inputs/actions) from computational mechanics \citep{crutchfield1989inferring}. 
Interestingly, the equivalence classes induced by $\epsilon$ are the minimal bisimulation of the world model $S_t=\bm H_{:t-1}$ (\autoref{lemma:funes}), and therefore serve as memory states of a transducer that generates the original interface (first shown in~\citep[Prop~2]{barnett2015computational}, alternative proof in~\autoref{app:proof_e_transducer_properties}).

\begin{proposition}
\label{prop:e_transducer_properties}
$(M_t, A_t, Y_t)$ with $M_t=\epsilon_t(\bm H_{:t-1})$ is a transducer presentation for $\mathcal{I}(\bm Y|\bm A)$.
\end{proposition}

The transducer with memory states given by $M_t = \epsilon_t(\bm H_{:t-1})$ resulting from \autoref{prop:e_transducer_properties} is known as the \textbf{$\boldsymbol{\epsilon}$-transducer} of the interface $\mathcal{I}(\bm Y|\bm A)$, and is unique up to isomorphism. 
The link between computational mechanics and other approaches such as predictive state representations was first noticed by \cite{zhang2019learning}, which explored it using a different computational structure instead of transducers. 
A salient feature of these approaches is that they can provide observable world models over fewer states than other methods~\citep{littman2001predictive}. 
Our next result strengthens this intuition by proving that the $\epsilon$-transducer yields the most efficient predictive world model possible (proof in~\autoref{app:Predictive Transducer Coarse-Graining}), which takes inspiration from and extends~\citep[Lemma~1]{barnett2015computational}.

\begin{theorem}
\label{theo:Predictive Transducer Coarse-Graining}
    If $R_t$ is a predictive world model of a transducer, then its minimal bisimulation is isomorphic to the $\epsilon$-transducer.
\end{theorem}

\begin{corollary}\label{corollary:e_transducer}
    The $\epsilon$-transducer is the minimal predictive model that generates a given interface.
\end{corollary}

These results reveal, for instance, that the predictive beliefs on all world models converge via bisimulating into the memory states of the $\epsilon$-transducer. 
In effect, while bisimulations of arbitrary transducers may not fully reduce world models (\autoref{sec:bisimulation}), bisimulations of predictive transducers necessarily do so (see~\autoref{app:general_vs_predictive}). An analogous result can be derived for postdictive beliefs and world models, which are reduced into a `time-shifted' $\epsilon$-transducer. This will be developed in a future publication.

\section{Reverse interpretability via retrodictive world models}
\label{sec:backwards_interp}

The results of the last section show that the $\epsilon$-transducer is a universal construction that distils the information that is relevant for predicting future events, which can be used to evaluate the extent to which agents can learn through a given interface. 
However, prediction alone does not exhaust the possible knowledge-driven activities that can involve an agent. 
This section investigates reversible and retrodictive world models, exploring new opportunities for agent interpretability.

\subsection{Reversible transducers}
\label{sec:reversible_transducers}

The kernel of a transducer is usually used to update a world model $S_t$ from $s_\tau$ to $s_{\tau+1}$. 
Interestingly, some transducers can be used to run things `backwards', so that the world state can be updated from $s_{\tau+1}$ to $s_\tau$ while generating the same interface.
This is formalised by the next definition.\footnote{This definition differs importantly from thermodynamically reversible transducers \citep{jurgens2020functional}.}

\begin{definition}
\label{def:rev_transducer}
A \textbf{reversible transducer} is a transducer  $\big(\mathcal{S}, \mathcal{A}, \mathcal{Y}, \kappa, \dist{p} \big)$ together with an additional Markov kernel $\kappa^R$ of the form $\{\kappa^R_t(y,s'|a,s): a \in \mathcal{A}, y \in \mathcal{Y},s,s' \in \mathcal{S}, t\in \mathbb{N}\}$ such that
\begin{align}\label{eq:def_reversible_transducer}
    p(\bm y_{:t},\bm s_{:t+1}| \bm a_{:})
    = p(s_{0})\prod_{\tau=0}^t \kappa_\tau(y_\tau,s_{\tau+1}|s_{\tau},a_{t})
    = p(s_{t+1}|\bm a_{:t})\prod_{\tau=0}^t \kappa_\tau^\textnormal{\text{R}}(y_\tau,s_\tau|a_{\tau},s_{\tau+1}).
\end{align}
\end{definition}

A reversible transducer can be run in reverse to produce the same interface. 
This can be used to analyse prior events that resulted in an undesirable world state $s^*_{t+1}$ after an agent executed actions $\bm a_{:t}$. This can be investigated via the distribution
\begin{equation}
    p(\bm s_{:t},\bm y_{:t}|\bm a_{:t},s^*_{t+1})
    =
    \frac{p(\bm y_{:t},\bm s_{:t},s^*_{t}| \bm a_{:})}{p(s^*_{t+1}|\bm a_{:t})}
    = 
    \kappa_\tau^\text{R}(y_t,s_t|a_{t},s^*_{t+1})
    \prod_{\tau=0}^{t-1} \kappa_\tau^\text{R}(y_\tau,s_\tau|a_{\tau},s_{\tau+1}),
\end{equation}
which allows to study how outputs lead to actions and identify tipping points in the world dynamics.

Unfortunately, not all transducers are reversible, as swapping past and future could break the condition of anticipation-free --- which is needed for a world model to yield a transducer (see~\autoref{app:non-reversible_transducers}). 
A necessary and sufficient condition for transducers to be reversed is provided next (proof in \autoref{app: Reversing Processes}).

\begin{theorem}\label{teo_reversing}
    A transducer is reversible for $\tau\leq T$ if and only if the dynamics of its memory state satisfy 
    $p(s_{\tau}|s_{\tau+1},\bm a_{:T}) = p(s_{\tau}|s_{\tau+1},a_{\tau})$ for all $\tau\in\{0,\dots,T\}$, with a reverse kernel given by
    \begin{equation}
\label{eq:reversed_kernel}
    \kappa^\textnormal{\text{R}}_\tau(y,s|a,\tilde{s}) 
    = 
    \frac{\Pr(S_\tau=s|A_\tau=a)}{\Pr(S_{\tau+1}=\tilde{s}|A_\tau=a)}
    \kappa_\tau(y,\tilde{s}|a,s).
\end{equation}
\end{theorem}

Although \autoref{eq:reversed_kernel} always leads to a valid kernel due to Bayes rule, this may not generate the same interface --- in fact, \autoref{eq:def_reversible_transducer} only holds when the conditions in \autoref{teo_reversing} are met. 
Interestingly, those conditions can be attained in a variety of ways. 
For example, memoryless transducers (see~\autoref{lemma:interface_memory}) are always reversible as $p(s_{t}|s_{t+1},\bm a_{:t}) = p(s_{t}|s_{t+1},a_{t}) = p(s_{t})$.
Also, consistent with results by~\cite{ellison2011information}, \emph{action-agnostic} transducers (i.e. hidden Markov models) can be shown to be always reversible (see~\autoref{app:hmm_reverse}). 
Finally, if the transducer is \textbf{action-counifilar} (i.e. if there exists $f$ such that $S_{t}=f(S_{t+1},A_t)$ can be deterministically updated)\footnote{This is a special case of \textbf{\textit{counifilar}} transducers, in which $S_{t}=f_t(S_{t+1}, A_t, Y_t)$ holds.} is also sufficient for reversibility, as such transducer satisfies $p(s_{\tau}|s_{\tau+1},\bm a_{:T})=\delta^{s_\tau}_{f(s_{\tau+1},a_\tau)}=p(s_{\tau}|s_{\tau+1},a_\tau)$. Examples of these conditions are illustrated in~\autoref{fig:CoUnifilarTransducers}.

\begin{figure*}
\includegraphics[width=\columnwidth]{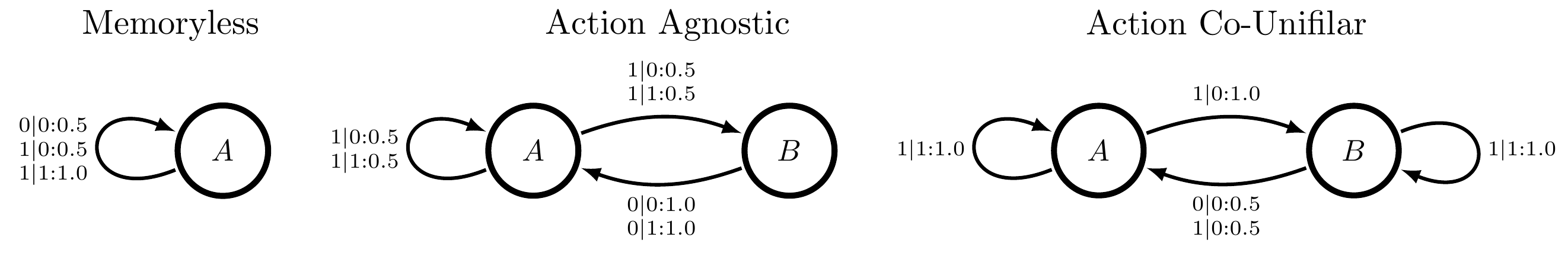}
\caption{
\centering
\small{Three examples of reversible transducers. Circles represent world states, and arrows represent transitions and their labels describe the associated actions and outputs. For instance, the label \texttt{$1|0\!\!:\!\!0.5$} on the edge from $s_0$ to $s_1$ indicates that $\Pr(S_{t+1}=s_1,Y_t=1|A_t=0,S_t=s_0)=0.5$.}}
\label{fig:CoUnifilarTransducers} 
\end{figure*}

\subsection{Retrodictive beliefs}
\label{sec:retro_beliefs}

The previous subsection showed how there are substantial restrictions on the reversibility of transducers. 
Even if an interface cannot be generated via a reversible transducer, there are still `retrodictive' constructions that can be used to investigate their dynamics. 
Retrodiction uses the future to learn about the past in the same way that prediction uses the past to learn about the future. 
Formal treatments of retrodiction include classic work in physics~\citep{watanabe1955symmetry} and filtering theory~\citep{jazwinski1970stochastic}, and in more recent years have been formalised in computational mechanics~\citep{ellison2009prediction} and category theory~\citep{parzygnat2023axioms, parzygnat2024reversing}.

Following these ideas, one can build \textbf{\textit{retrodictive Bayesian beliefs}} (or mixed states) of a world model $S_t$ as distributions over $\mathcal{S}$ given by 
$\dist{r}_{t}(s_0) \coloneq \Pr(S_0=s_0|\bm H_{:t}=\bm h_{:t})$. 
These beliefs provide an analogue of the backward pass of Bayesian smoothing~\citep{jazwinski1970stochastic}, in the same way that the predictive and postdictive beliefs of input-Moore transducers correspond to different steps of Bayesian filtering (\autoref{prop:postdictive_stuff}). 
However, in contrast with predictive Bayesian beliefs which can always be faithfull (\autoref{prop:predictive_beliefs}), retrodictive beliefs may not be able to generate the same interface.

In order to study the dynamics of retrodictive beliefs, we introduce the \textbf{\textit{bi-directional mixed-state matrix}} (BDMSM) of an action-outcome sequence $\rho(\bm y_{0:t},\bm a_{0:t})$ as the $|\mathcal{S}| \times |\mathcal{S}|$ matrix given by 
\begin{align}
    \rho(\bm y_{0:t},\bm a_{0:t}) 
    \coloneq
    \sum_{s_0,s_{t+1}\in\mathcal{S}}p(s_0,s_{t+1}|\bm y_{0:t},\bm a_{0:t})\bm e_{s_{t+1}} \bm e^\intercal_{s_0}~.
\end{align}
The BDMSM allows to calculate retrodictive beliefs and their dynamics (proof in~\autoref{app:proof_retro_mix}).
\begin{theorem}\label{teo:retro_MSP}
    Given a world model $S_t$, its BDMSM, predictive Bayesian beliefs $\dist{b}_\tau$ and retrodictive Bayesian beliefs $\dist{r}_\tau$ can be calculated as
    \begin{align}
        \rho(\bm y_{0:\tau},\bm a_{0:\tau})
        =
        \frac{T^{(\bm y_{0:\tau}|\bm a_{0:\tau})}  \rho_0 }
        { \bm 1^\intercal \cdot T^{(\bm y_{0:\tau}|\bm a_{0:\tau})}  \rho_0 \cdot\bm 1}, 
        \quad 
        \dist{b}_\tau =\rho(\bm y_{0:\tau},\bm a_{0:\tau}) \cdot \bm 1, 
        \quad\text{and}\quad 
        \dist{r}_{\tau} = \rho(\bm y_{0:\tau}, \bm a_{0:\tau})^\intercal \cdot \bm 1,
        \nonumber
    \end{align}
    where $T^{(y_{\tau:\tau'}|a_{\tau:\tau'})} = \prod_{j=\tau}^{} T_j^{(y_{j}|a_{j})}$
    and $\rho_t = \sum_{s_t} p(s_t)\bm e_{s_{t}} \bm e^\intercal_{s_t} $ 
    is a diagonal matrix. 
\end{theorem}
\begin{corollary}
    \label{cor:smoothing}
    The forward-time update of the BDMSM is given by
    \begin{align}
        \rho(\bm y_{0:\tau+1},\bm a_{0:\tau+1})
        =
        \frac{T^{(y_{\tau+1}|a_{\tau+1})}}
        {\bm 1^\intercal \cdot T^{(y_{\tau+1}|a_{\tau+1})}\rho(\bm y_{0:\tau},\bm a_{0:\tau}) \cdot \bm 1}
        \cdot
        \rho(\bm y_{0:\tau},\bm a_{0:\tau}),
        \end{align}
        while the reverse-time update is 
        \begin{align}
        \rho(\bm y_{-1:\tau},\bm a_{-1:\tau})
        =
        \rho(\bm y_{0:\tau},\bm a_{0:\tau})
        \cdot
        \frac{\rho^{-1}_0T^{(y_{-1}|a_{-1})}\rho_{-1}}{\bm 1^\intercal \cdot\rho(\bm y_{0:\tau},\bm a_{0:\tau})\rho_0^{-1}T^{(y_{-1}|a_{-1})}\rho_{-1} \cdot \bm 1}.
    \end{align}
\end{corollary}

Retrodictive beliefs can be used to infer the most likely past states of the world given a sequence of future actions and outcomes. 
This could lead, for instance, to identifying the origins of specific behavioural patterns exhibited by an AI agent, 
which can in turn be used to characterise favourable or dangerous initial conditions via counterfactual reasoning~\citep{karimi2021algorithmic}.

\section{Conclusion}


This paper investigated the fundamental limits that shape the usage of world models as tools to evaluate AI agents. 
This follows recent proposals to use world models not as tools for the agent (as in standard model-based reinforcement learning), but as tools for the scientist in charge of evaluating its safety and reliability~\citep{dalrymple2024towards}.   
By formalising these ideas via principles from computational mechanics, this approach led to a series of proposals for how to assess AI agents that require no assumptions about an agent's policy, architecture, or capabilities, being broadly applicable to systems regardless of how they were designed or trained. 
This framework revealed fundamental limits, challenges, and opportunities inherent to world modelling, leading to actionable guidelines that can inform core design choices instrumental for effective agent evaluation (see~\autoref{fig:diagram}).

Our framework revealed a fundamental trade-off between the efficiency and interpretability of world models. 
Generalised transducers were found to generate the most efficient implementations, but these come at the cost of inducing quasi-probabilities --- yielding opaque world models that cannot be sampled.~\footnote{This is reminiscent of the notion of Kantian noumena, which suggests that things-in-themselves are beyond knowledge.} 
Our results also revealed that the $\epsilon$-transducer, a generalisation of the geometric belief structure recently found in the residual stream of transformers~\citep{shai2025transformers}, yields the unique minimal world model that could be calculated by an agent in real time. 
The uniqueness of the $\epsilon$-transducer implies that the refinement of the beliefs of any optimal predictive agent must eventually reach this model, regardless of the world model the agent uses. Thus, the $\epsilon$-transducer can be seen as encapsulating all the predictive information that is available for agents, and hence establishes what is learnable about an environments through a particular interface.

We also introduced retrodictive world models as tools to investigate the origins of undesirable events or behaviours. 
These models allow retrospective analyses that could, for instance, identify `danger zones' that are likely to lead to undesirable future outcomes. 
This view complements standard interpretability approaches, which typically assess agents via their capabilities to predict and plan with respect to future events~\citep{nanda2023emergent, gurnee2023language, shai2025transformers}.

While this work focused on the fundamental limits of world modelling under the dictum of perfect reconstruction, future work may relax this constraint by employing notions such as approximate homomorphisms~\citep{taylor2008bounding} or bisimulation~\citep{girard2011approximate}, rate-distortion trade-offs~\citep{marzen2016predictive}, or other approaches~\citep{subramanian2022approximate}. 
Another promising direction to enable efficient modelling is to exploit the compositional structure of the world~\citep{lake2023human,elmoznino2024complexity,baek2025dreamweaver,fu2025evaluating}.

The approach taken here complements the substantial body of work that employs world models to improve the performance of agents in model-based reinforcement learning~\citep{ha2018recurrent, hafner2019dream, hafner2023mastering,hansen2023td}, and also on representations from the point of view of the agent (see~\citep{ni2024bridging} and references within). 
In fact, the formalism presented here provides a unified framework for reasoning about both (i) models that represent physical processes external to the agent and (ii) models that describe knowledge-gathering processes internal to the agent~\citep{kaelbling1998planning, biehl2022interpreting, virgo2023unifilar}. 
Furthermore, the relationship found between predictive and postdictive machines in I-O Moore transducers and Bayesian and Kalman filtering sheds new light into the mechanisms supporting these well-established procedures. 
Moreover, the formalism of belief transducers opens several interesting avenues for future work, including the investigation of more general belief update dynamics on, for example, curved statistical manifolds~\citep{morales2021generalization, aguilera2024explosive}.

Overall, the ideas put forward here establish new bridges between related subjects in reinforcement learning, control theory, and computational mechanics, which we hope may serve as a Rosetta stone for navigating across these literatures.  
These new insights also have interesting implications for cognitive and computational neuroscience~\citep{matsuo2022deep}, particularly pertaining the formal characterisation of the internal world (`umwelt') of an agent~\citep{von1909umwelt, ay2015umwelt, baltieri2025bayesian}, which will be explored in future work.

\appendix

\section*{Acknowledgments}

The authors thank Lionel Barnett, Martin Biehl, Chris Buckley, Matteo Capucci, James Crutchfield, Alexander Gietelink Oldenziel, Adam Goldstein, Alexandra Jurgens, Vanessa Kosoy, Sarah Marzen, Paul Riechers, Anil Seth, Adam Shai, and Lucas Teixera for inspiring discussions and useful feedback. 
The work of F.R. and A.B. has been supported by UK ARIA's Safeguarded AI programme. 
F.R. has also been supported by the PIBBSS Affiliateship programme. 
M.B. was supported by JST, Moonshot R\&D, Grant Number JPMJMS2012.

\bibliography{main}
\bibliographystyle{rlj}

\beginSupplementaryMaterials

\section{Example of an interface and multiple world models}
\label{app:example}

Let us present an example to illustrate the notions of interface and world models. This example showcases how a single interface can be generated by various world models with different properties.

Consider a scenario in which a robot is manipulating a deck of cards. 
This setting can be described via a world model that can adopt $|\mathcal{S}|=13!\approx 6\times 10^{10}$ possible states, corresponding to the possible arrangements of the deck. 
At every time point, the robot can take two possible actions: either it puts the front card in the back ($\alpha_1$), or it shuffles the deck ($\alpha_2$). Thus, the set of actions available to the robot is $\mathcal{A}=\{\alpha_1,\alpha_2\}$. 
Additionally, at every time point the robot can observe the card that is on top of the deck. However, we will assume that the sensory apparatus of the robot is not capable of reading the number or the suit of the card, but only its colour. Hence, the possible outcomes of this scenario for the robot are $\mathcal{Y}=\{\texttt{black},\texttt{red}\}$.

In this scenario, the interface of the robot is constituted by a collection of probability distributions of the form $p(\bm y_{:t}|\bm a_{:t})$ relating sequences of actions with sequences of outcomes. Furthermore, if one tracks the state of the deck at time $t$ via the variable $S_t$, this results in a transformer $(S_t, A_t, Y_t)$ (see \autoref{def:transducer}). If we wanted to implement this transducer on a computer, specifying its kernel would require a substantial amount of memory due to the large number of possible world states.

Following the considerations made in \autoref{sec:scenario_and_approach}, one could instead forget about the fact that there is an underlying deck of cards and focus just on the sequences of colours that the agent records.
By doing this, one may notice that, given a sequence of actions and outcomes $(\bm a_{:t},\bm y_{:t})$, the only information that is relevant to predict the next outcome $y_{t+1}$ is the number of red and black cards observed since the last time the agent took action $\alpha_2$ (shuffling the deck).
If one tracks this information at time $t$ via the variable $M_t$, then \autoref{prop:e_transducer_properties} guarantees that $(M_t,A_t,Y_t)$ provides an alternative transducer presentation of the original interface. 
Note that this is an `epistemic' world model that reflects the agent's state of knowledge (as described in \autoref{sec:epistemic_worlds}), contrasting with $S_t$ which reflects an objective physical process taking place `out there'. 
Interestingly, $M_t$ uses only roughly $13^2$ states instead of $13!$ states, requiring substantially fewer memory resources. 
Furthermore, $M_t$ reflects all the relevant information that an agent with this particular interface (i.e. limited to recognising colours) could ever want to take into account in order to make informed actions in this scenario.
Therefore, this new world model is not only more memory efficient, but also reveals what information an agent of this kind could and should learn.

Before concluding, let us add some considerations related to the ideas explored in the latest part of the manuscript. 
\autoref{sec:backwards_interp} studies world models that `run backwards' --- i.e. can be updated in reverse time. 
The interface chosen in this example does not allow for such a reversible presentation, as the combinations of shuffling ($\alpha_2$) and card flipping ($\alpha_1$) lead to world dynamics that violate the conditions outlined in \autoref{teo_reversing}.
One could attain a reversible interface if we considered a different set of actions, for example $\mathcal{A}'=\{\alpha_1,\alpha_3$\} with $\alpha_3$ corresponding to the robot taking the card in the back and putting that on top of the deck.
Indeed, the world dynamics resulting from such a set of actions  
do satisfy the sufficient condition for reversibility discussed at the end of \autoref{sec:reversible_transducers}.

\section{Sufficient statistics}
\label{app:sufficient_stats}

Given the importance of the notion of sufficient statistics in this work, we use this appendix to provide an account of its origins and significance.

Consider a random vector $\bm X = (X_1,\dots,X_n)\in\mathcal{X}^n$ that follows a distribution with parameter $\theta\in\Theta$, and a `statistic' $T(\cdot)$ (that is, a mapping $T:\mathcal{X}^n\to\mathbb{R}$). Following~\cite{fisher1922mathematical}, $Y=T(\bm X)$ is a \emph{classical/frequentist sufficient statistic} for $\bm X$ w.r.t. $\theta$ if the value of $\Pr_\theta(\bm X = \bm x | Y=y)$ is the same $\forall\theta\in \Theta$~\citep{casella2002statistical}. This means that the information given by $\bm X$ that is not in $Y$ is irrelevant to estimating the value of $\theta$. 

Another approach to statistical sufficiency due to ~\cite{kolmogorov1942determination}, which can be called \emph{strong bayesian statistical sufficiency}, states that $Y$ is sufficient for $\bm X$ w.r.t. $\theta$ if $\bm X$ is statistically independent of $\theta$ given $Y$ for any prior distribution over $\theta$. Strong Bayesian sufficiency can be shown to imply classical sufficiency, but the converse does not necessarily hold~\citep{blackwell1982bayes}.

A useful generalisation of the above condition, which we simply call \emph{(weak) Bayesian statistical sufficiency}, follows Kolmogorov's condition just for a given distribution of $\theta$~\citep{cover2012elements}.
In particular, given two random variables $X$ and $Y$, a statistic $T=f(X)$ is said to be a \emph{Bayesian sufficient statistic for $X$ w.r.t. $Y$} if $X$ is statistically independent of $Y$ given $T$.
In terms of Shannon's mutual information, this corresponds to the condition $\Pr(X=x|Y=y,T=t)=\Pr(X=x|T=t)$. 
This is equivalent to the information-theoretic condition $I(X;Y|T)=0$, which states that $X$ and $Y$ share no information that is not given by $T$~\citep{cover2012elements}.
This is the definition of sufficient statistics that we use in this work.

Another way to think about sufficient statistics is to notice that if $T=f(X)$ for some mapping $f$ then $T-X-Y$ is a Markov chain.
Then, thanks to the data processing inequality, $I(Y;X)\geq I(Y;T)$ as `processing' $X$ into $T$ cannot increase its information about $Y$~\citep{cover2012elements}.
Interestingly, the equality $I(Y;X) = I(Y;T)$ is attained if an only if $X-T-Y$ is also a Markov chain, which corresponds to when $T$ is a sufficient statistic. In summary, sufficient statistics are related to optimal (i.e. lossless) data processing~\citep{kullback1997information}.

Sufficient statistics always exist --- in particular, $X$ is always sufficient for itself. The search for optimal but also efficient statistics leads to the idea of minimal sufficiency: a sufficient statistic $S$ is minimal if for all other sufficient statistic $T$ exists a function $f(\cdot)$ such that $S=f(T)$~\citep{lehmann2012completeness}, or equivalently, the following Markov chain holds: $S-T-X-Y$. From an information-theoretic point of view, a minimal sufficient statistic is the sufficient statistic of minimal entropy, hence providing the most parsimonious representation of the relevant information. 
Minimal sufficient statistics exist for a wide range of settings~\cite[Sec. 1.6]{lehmann2006theory}, and are unique up to isomorphisms (i.e. re-labelling). Moreover, the minimal sufficient statistics of $X$ w.r.t. $Y$ can be built explicitly, built as the partition induced by the following equivalence relation~\cite[Def. 2]{asoodeh2014notes}: 
\begin{equation}
x\sim x' 
\quad \text{iff} \quad
\forall y\in\mathcal{Y}:~ p(y|x) = p(y|\bm x').
\end{equation}
Note the similarities between this way of building minimal sufficient statistics, bisimulation (\autoref{def:bisimulation}), and the construction of the $\epsilon$-transducer via the equivalence relation in~\autoref{eq:epsilon}.

\section{Proof of \autoref{lemma:world_model}}
\label{app:proof_of_factorisation_forward}

\begin{proof}
Let us first prove that if $S_t$ is a world model for the interface $\mathcal{I}(\bm Y|\bm A)$, then \autoref{eq:world_model} holds. 
Using property~(2) of world models, together with the fact that the interface is anticipation-free, one can show that
\begin{align}
    p(\bm y_{:\tau},\bm s_{:\tau+1}|\bm a_{:})
    &=
    p(\bm y_{:\tau}|\bm a_{:}) p(\bm s_{:\tau+1}|\bm y_{:\tau},\bm a_{:}) \nonumber \\
    &=
    p(\bm y_{:\tau}|\bm a_{:\tau}) p(\bm s_{:\tau+1}|\bm y_{:\tau},\bm a_{:\tau}) \nonumber \\
    &=
    p(\bm y_{:\tau},\bm s_{:\tau+1}|\bm a_{:\tau}),
\end{align}
which holds for all $\tau\in\mathbb{N}$. Then, one can use this equality recursively to derive the following:
\begin{align}
    p(\bm y_{:t},\bm s_{:t+1}|\bm a_{:})
    &= 
    p(\bm y_{:t},\bm s_{:t+1}|\bm a_{:t}) \nonumber \\
    &= 
    p(\bm y_{:t-1},\bm s_{:t}|\bm a_{:t}) p(y_{t},s_{t+1}|\bm y_{:t-1},\bm s_{:t},\bm a_{:t}) \nonumber \\
    &=
    p(\bm y_{:t-1},\bm s_{:t}|\bm a_{:t-1})
    p(y_{t},s_{t+1}|\bm y_{:t-1},\bm s_{:t},\bm a_{:t}) \nonumber \\
    &=
    p(\bm y_{:t-2},\bm s_{:t-1}|\bm a_{:t-1})
    \prod_{\tau=t-1}^t p(y_{\tau},s_{\tau+1}|\bm y_{:\tau-1},\bm s_{:\tau},\bm a_{:\tau}) \nonumber \\
    &= \ldots \nonumber \\
    &=
    p(s_0)\prod_{\tau=0}^t p(y_\tau,s_{\tau+1}|\bm y_{:\tau-1},\bm s_{:\tau},\bm a_{:\tau}).\label{eq:x}
\end{align}
Note that, in the last step, $p(s_0|a_0)=p(s_0)$ follows by applying property (2) for $t=0$. Now, using property (1) one can show that
\begin{align}
    p(y_\tau,s_{\tau+1}|\bm y_{:\tau-1},\bm s_{:\tau},\bm a_{:\tau}) 
    &=
    p(y_\tau|\bm y_{:\tau-1},\bm s_{:\tau},\bm a_{:\tau}) 
    p(s_{\tau+1}|\bm y_{:\tau},\bm s_{:\tau},\bm a_{:\tau}) \nonumber \\
    &=
    p(y_\tau|s_{\tau},a_{\tau}) 
    p(s_{\tau+1}|\bm y_{:\tau},\bm s_{:\tau},\bm a_{:\tau}).\label{eq:xx}
\end{align}
The desired result follows from putting together \autoref{eq:x} and \autoref{eq:xx}. 

For the converse, let's show that if \autoref{eq:world_model} holds, then $S_t$ satisfies the two properties of world models as given in \autoref{def:world_model}. 
The first property can be proven directly as follows:
\begin{align}
    p(y_{t}|\bm s_{:t},\bm a_{:t},\bm y_{:t-1})
    &=
    \frac{\sum_{s_{t+1}}
    p(\bm y_{:t},\bm s_{:t+1}|\bm a_{:t})}
    {\sum_{s_{t+1},y_t} p(\bm y_{:t},\bm s_{:t+1}|\bm a_{:t})} \nonumber \\
    &=
    \frac{\sum_{s_{t+1}}
    p(s_0)\prod_{\tau=}^{t}
    p(y_\tau|s_{\tau},a_{\tau}) p(s_{\tau+1}|\bm s_{:\tau},\bm h_{:\tau})}
    {\sum_{s_{t+1},y_t} 
    p(s_0)\prod_{\tau=0}^{t}
    p(y_\tau|s_{\tau},a_{\tau}) p(s_{\tau+1}|\bm s_{:\tau},\bm h_{:\tau})}
    \nonumber \\
    &=
    \frac{\sum_{s_{t+1}}
    p(y_t|s_{t},a_{t}) p(s_{t+1}|\bm s_{:t},\bm h_{:t})}
    {\sum_{s_{t+1},y_t} 
    p(y_t|s_{\tau},a_{t}) p(s_{t+1}|\bm s_{:t},\bm h_{:t})}
    \nonumber \\
    &=
    p(y_{t}|s_{t},a_{t}).
\end{align}
Similarly, the second property can be proven as follows:
\begin{align}
    p(\bm s_{:t}| \bm h_{:t-1},\bm a_{t:})
    &=
    \frac{p(\bm s_{:t},\bm y_{:t-1}|\bm a_{:})}
    {p(\bm y_{:t-1}| \bm a_{:})}
    \nonumber \\
    &= 
    \frac{p(s_0)\prod_{\tau=}^{t-1}
    p(y_\tau|s_{\tau},a_{\tau}) p(s_{\tau+1}|\bm s_{:\tau},\bm h_{:\tau})}
    {\sum_{\bm s_{:t}} p(s_0)\prod_{\tau=}^{t-1}
    p(y_\tau|s_{\tau},a_{\tau}) p(s_{\tau+1}|\bm s_{:\tau},\bm h_{:\tau})} \nonumber \\
    &\stackrel{(a)}{=}
    p(\bm s_{:t}| \bm h_{:t-1}).
\end{align}
Above, (a) follows from the fact that the variables $\bm a_{t:}$ do not appear in the previous expression.
\end{proof}

\section{Alternative characterisations of a transducer}
\label{app:proof_transducer_faces}

\autoref{lemma:transducer_as_world_model} characterises transducers as world models. However, this characterisation is not the most useful when investigating if a given process $S_t$ qualifies as the memory state of a transducer. Here we provide two alternative characterisations of a transducer that are better suited to those tasks.

\begin{lemma}
\label{lemma:transducer_faces}
     The process $S_t$ provides a memory state for a transducer presentation of an anticipation-free interface $\mathcal{I}(\bm Y|\bm A)$ if and only if one of the following conditions hold:
     \begin{enumerate}
     \item 
     $p(s_0|\bm a_:) = p(s_0)$ and $p(s_{t+1},y_t | \bm s_{:t}, \bm h_{:t-1}, \bm a_{t:})
     = 
     p(s_{t+1},y_t | s_{t}, a_{t} )$ for all $t\in\mathbb{N}$.
     \item 
$I(\bm S_{t_i+1:t},\bm Y_{t_i:t-1};\bm A_{t:}|\bm A_{t_i:t-1},S_{t_i})=I(\bm {S}_{t+1:},\bm Y_{t:};\bm{Y}_{:t-1},\bm{S}_{:t-1},\bm{A}_{:t-1}|\bm{A}_{t:},S_t)=0$ for all $t_i,t\in\mathbb{N}$ with $t_i\leq t$.
\end{enumerate}     
\end{lemma}

\begin{proof}

We prove these equivalences in two steps.

\subsection*{Step 1: Equivalence between condition (1) and \autoref{lemma:transducer_as_world_model}}

Let's first prove that \autoref{lemma:transducer_as_world_model} imply condition (1). 
By \autoref{lemma:world_model}, if $S_t$ is a world model then
    $p(s_{t+1},y_t | \bm s_{:t}, \bm h_{:t-1}, \bm a_{t:})
    =
    p(y_t | s_{t}, a_{t})
    p(s_{t+1} | \bm s_{:t}, \bm h_{:t})$.   
Combining this with 
$p(s_{t+1} | \bm s_{:t}, \bm h_{:t}) 
    = 
p(s_{t+1} | s_{t}, h_{t})$ (\autoref{lemma:transducer_as_world_model}), it is clear that 
$p(s_{t+1},y_t | \bm s_{:t}, \bm h_{:t-1}, \bm a_{t:})
=
p(s_{t+1},y_t | s_{t}, a_{t})$.

To prove the converse, let us now show that condition (1) guarantees that $S_t$ is a world model that satisfies $p(s_{t+1} | \bm s_{:t}, \bm h_{:t}) 
    = 
p(s_{t+1} | s_{t}, h_{t})$. 
Property (1) of world models can be proven as follows:
    \begin{align}
        p(y_{t}| \bm s_{:t},\bm a_{:t},\bm y_{:t-1})
        = 
        \sum_{s_{t+1}} p(s_{t+1}, y_{t}| \bm s_{:t},\bm h_{:t-1},a_{t})
        \stackrel{(a)}{=} 
        \sum_{s_{t+1}} p(s_{t+1}, y_{t}| s_{t},a_{t})
        = 
        p(y_{t}| s_{t},a_{t}),
    \end{align}
where (a) uses condition (1). 
Property (2) follows from
    \begin{align}
    p(\bm s_{:t}| \bm h_{:t-1},\bm a_{t:}) 
    &=
    \frac{p(\bm s_{:t}, \bm y_{:t-1} | \bm a_{:})}
    {p(\bm y_{:t-1} | \bm a_{:}) } \nonumber \\
    &=
    \frac{\prod_{\tau=0}^{t-1} p(s_{\tau+1}, y_{\tau} | \bm s_{:\tau}, \bm y_{:\tau-1}, \bm a_{:})}
    {p(\bm y_{:t-1} | \bm a_{:}) } \nonumber \\
    &\stackrel{(b)}{=}
    \frac{\prod_{\tau=0}^{t-1} p(s_{\tau+1}, y_{\tau} | \bm s_{:\tau}, \bm y_{:\tau-1}, \bm a_{:t-1})}
    {p(\bm y_{:t-1} | \bm a_{:t-1}) } \nonumber \\
    &=
    \frac{p(\bm s_{:t}, \bm y_{:t-1} | \bm a_{:t-1})}
    {p(\bm y_{:t-1} | \bm a_{:t-1})} \nonumber \\
    &= p(\bm s_{:t}| \bm h_{:t-1}),
    \end{align}
    where (b) is using condition (1) and the fact that the interface is anticipation free. 
    Finally, the Markovianity of state dynamics can be proven as follows:
    \begin{align}
        p(s_{t+1}|\bm s_{:t},\bm h_{:t}) 
        =
        \frac{ p(s_{t+1},y_t|\bm s_{:t}, a_t, \bm h_{:t-1})}
        {\sum_{s_{t+1}} p(s_{t+1},y_t|\bm s_{:t}, a_t,\bm h_{:t-1})}
        = 
        \frac{ p(s_{t+1},y_t|s_t, a_{t})}
        {\sum_{s_{t+1}} p(s_{t+1},y_t|s_t, a_{t})}
        =
        p(s_{t+1}|s_t,h_t).
    \end{align}

\subsection*{Part 2: Equivalence between conditions (1) and (2)}
\label{proof:Information Conditions}

Let's first show that condition (2) implies condition (1). 
For this, let's first note that in general if $I(A;B|C)=0$ holds for some variables $A,B$, and $C$, then $p(a|c)=p(a|b,c)$.
Thus, the condition 
$I(\bm S_{t_i+1:t},\bm Y_{t_i:t-1};\bm A_{t:}|\bm A_{t_i:t-1},S_{t_i})=0$ implies that 
\begin{align}
    p(\bm s_{t_i+1:\tau},\bm y_{t_i:\tau-1}|\bm a_{t_i:},s_{t_i})  
    &=
    p(\bm s_{t_i+1:\tau},\bm y_{t_i:\tau-1}|\bm a_{t_i:\tau-1},s_{t_i}),
\end{align}
holding for all $\tau\in\mathbb{N}$ with $t_i\leq\tau\leq t$. 
Similarly, $I(\bm {S}_{t+1:},\bm Y_{t:};\bm{Y}_{:t-1},\bm{S}_{:t-1},\bm{A}_{:t-1}|\bm{A}_{t:},S_t)=0$ implies that 
\begin{align}
    p(\bm s_{\tau+1:t+1},\bm y_{\tau:t}|\bm y_{:\tau-1},\bm s_{:\tau},\bm a_{:}) 
    &= 
    p(\bm s_{\tau+1:t+1},\bm y_{\tau:t}|\bm a_{\tau:},s_{\tau}),
\end{align}
holding for all $\tau\in\mathbb{N}$ with $t_i\leq\tau\leq t$.
Note that $S_{0}$ is an element of the past $\bm S_{:t-1}$, so we can multiply these together to obtain
\begin{align}
    p(\bm s_{t_i+1:\tau},\bm y_{t_i:\tau-1}|&\bm a_{t_i:\tau-1},s_{t_i})
    p( \bm s_{\tau+1:t+1},\bm y_{\tau:t}|\bm a_{\tau:},s_{\tau}) \nonumber \\
    &=
    p(\bm s_{t_i+1:\tau},\bm y_{t_i:\tau-1}|\bm a_{t_i:},s_{t_i})
    p(\bm s_{\tau+1:t+1},\bm y_{\tau:t}|\bm y_{t_i:\tau-1},\bm s_{t_i:\tau},\bm a_{t_i:}) \nonumber \\ 
    &= 
    p(\bm s_{t_i+1:t+1},\bm y_{t_i:t}|\bm a_{t_i:},s_{t_i}).
\end{align}
Using this relation recursively, one can find that 
\begin{align}
    p(\bm s_{:t+1},\bm y_{:t},|\bm a_{:},s_{0})
    &\stackrel{(c)}{=} 
    p(\bm s_{:1},y_0|a_0,s_0)
    p(\bm s_{2:t+1},\bm y_{1:t}|\bm a_{1:},s_{1}) \nonumber \\
    &\stackrel{(d)}{=}
    p(s_1,y_0|a_0,s_0)
    p(s_{2},y_{1}|s_{1},a_{1})
    p(\bm s_{3:t+1},\bm y_{2:t}|s_{2},\bm a_{2:}) \nonumber \\
    &=
    \ldots \nonumber \\
    &=
    \prod_{\tau=0}^t p(s_{\tau+1},y_{\tau}|s_\tau,a_{\tau}),
\end{align}
where (c) is obtained by using $t_i=0$ and $\tau=1$, (d) by using $t_i=1$ and $\tau=2$, and so on. 
By comparing with \autoref{lemma:transducer_as_world_model}, this means that condition (2) implies condition (1).

Let us now show that condition (1) implies condition (2). 
Condition (1) implies that 
\begin{align}
p(\bm s_{t_i+1:t_f},\bm y_{t_i:t_f-1},|\bm a_{t_i:},s_{t_i}) 
&= 
\prod_{\tau=t_i}^{t_f} p(s_{t+1},y_{t}|s_{t},a_{t}) \nonumber \\
&= 
\left(\prod_{j=t}^{t_f} p(s_{j+1},y_{j}|s_{j},a_{j})\right)
\left(\prod_{k=t_i}^{t} p(s_{k+1},y_{k}|s_{k},a_{k})\right) \nonumber \\
&=
p(\bm s_{t+1:t_f},\bm y_{t:t_f-1}|\bm a_{t:t_f-1},s_{t})
p(\bm s_{t_i+1:t},\bm y_{t_i:t-1}|\bm a_{t_i:t-1},s_{t_i}).
\label{eq:deriv}
\end{align}
This can be used to show that
\begin{align}
p(\bm s_{t_i+1:t},\bm y_{t_i:t-1}|\bm a_{t_i:t-1},s_{t_i})
&= 
\sum_{\substack{\bm s_{t+1:t_f}\\\bm y_{t:t_f-1}}}
p(\bm s_{t+1:t_f},\bm y_{t:t_f-1}|\bm a_{t:t_f-1},s_{t})
p(\bm s_{t_i+1:t},\bm y_{t_i:t-1}|\bm a_{t_i:t-1},s_{t_i})
\nonumber \nonumber \\
&=    
\sum_{\substack{\bm s_{t+1:t_f}\\\bm y_{t:t_f-1}}}
p(\bm s_{t_i+1:t_f},\bm y_{t_i:t_f-1},|\bm a_{t_i:},s_{t_i}) \nonumber \\
&=    
p(\bm s_{t_i+1:t},\bm y_{t_i:t-1},|\bm a_{t_i:},s_{t_i}),
\end{align}
which implies 
$I(\bm S_{t_i+1:t}, \bm Y_{t_i:t-1}; \bm A_{t:}| \bm A_{t_i:t-1},S_{t_i})=0$. 
To prove the second information equality, one can divide both sides of \autoref{eq:deriv} by $p(\bm s_{t_i+1:t},\bm y_{t_i:t-1}|\bm a_{t_i:t-1},s_{t_i})$ to obtain 
\begin{align}
p(\bm s_{t+1:t_f},\bm y_{t:t_f-1}|\bm a_{t:t_f-1},s_{t})
&=
\frac{p(\bm s_{t_i+1:t_f},\bm y_{t_i:t_f-1},|\bm a_{t_i:},s_{t_i})}
{p(\bm s_{t_i+1:t},\bm y_{t_i:t-1}|\bm a_{t_i:t-1},s_{t_i})} \nonumber \\
&= 
p(\bm s_{t+1:t_f},\bm y_{t:t_f-1},|\bm y_{t_i:t-1},\bm s_{t_i:t},a_{t_i:}).
\end{align}
Given that $t_i$ and $t_f$ are arbitrary, this implies that $I(\bm S_{t+1:},\bm Y_{t:}; \bm Y_{:t-1},\bm S_{:t},\bm A_{:t-1} |\bm A_{t:},S_t)=0$.
\end{proof}

\section{Proof of \autoref{lemma:interface_memory}}
\label{app:interface_memory}

\begin{lemma}
An interface is fully observable if and only if $p(y_{t+1}|\bm y_{:t},\bm a_:) = p(y_{t+1}| y_t,a_t)$, and is memoryless if and only if $p(\bm y_{:t}|\bm a_:) = \prod_{\tau=0}^t p(y_\tau|a_\tau)$. 
\end{lemma}

\begin{proof}
To prove the first part of the lemma, one can use condition (1) in \autoref{lemma:transducer_faces} which implies that 
$S_t=Y_t$ yields a transducer if and only if 
        $p(y_\tau,y_{\tau+1}|\bm y_{:\tau}, \bm a_{:\tau}) 
        = p(y_{\tau},y_{\tau+1}|y_{\tau}, a_\tau)
        = p(y_{\tau+1}|y_{\tau}, a_\tau)$. 
To prove the second part of the lemma, note that an interface satisfies $p(\bm y_{:t}|\bm a_:) = \prod_{\tau=0}^t p(y_\tau|a_\tau)$ if and only if $S_t=0$ yields a factorisation of $p(\bm y_{:t}|\bm a_:)$ as in~\autoref{eq:transducer}. This shows that $S_t=0$ is the state of a transducer presentation of $\mathcal{I}(\bm Y|\bm A)$ if and only if the interface is memoryless.
    
\end{proof}

\section{Relationship between transducers and POMDPs}
\label{app:POMDPs}

A POMDP is a tuple $(\mathcal{S}, \mathcal{A}, \mathcal{O}, \tau, \mu, \rho)$ in which $\mathcal{S}$ correspond to states of the world, $\mathcal{A}$ the action space, $\mathcal{O}$ the observation space, and the probability kernels $\tau : \mathcal{S} \times \mathcal{A} \to \Delta(\mathcal{S})$, $\mu: \mathcal{S} \to \Delta(\mathcal{O})$, and $\rho: \mathcal{S} \times \mathcal{A} \to \Delta(\mathbb{R})$ specify the world dynamics, observation map, and reward function~\citep{kaelbling1998planning}. 
Under a POMDP, the joint dynamics satisfy \autoref{eq:transducer}, which --- thanks to condition (1) in \autoref{lemma:transducer_faces} --- is sufficient to show that the POMDP induces a transducer. This, together with~\autoref{lemma:transducer_as_world_model}, implies that the process $S_t$ in a POMDP is a world model, in the sense that it satisfies the conditions in \autoref{def:world_model}. 

Note also that the standard presentaiton of POMDPs correspond to a transducer whose kernel allows the following factorisation:
\begin{equation}
    p(s_{t+1},y_t|s_t,a_t) 
    = \tau(s_{t+1}|s_t,a_t) \mu(o_t|s_t) \rho(r_t|s_t,a_t).
\end{equation}
This corresponds to a I-O Moore transducer, as defined in \autoref{sec:transducer_types}. The different types of transducers are illustrated in~\autoref{fig:Mealy_vs_Moore}.

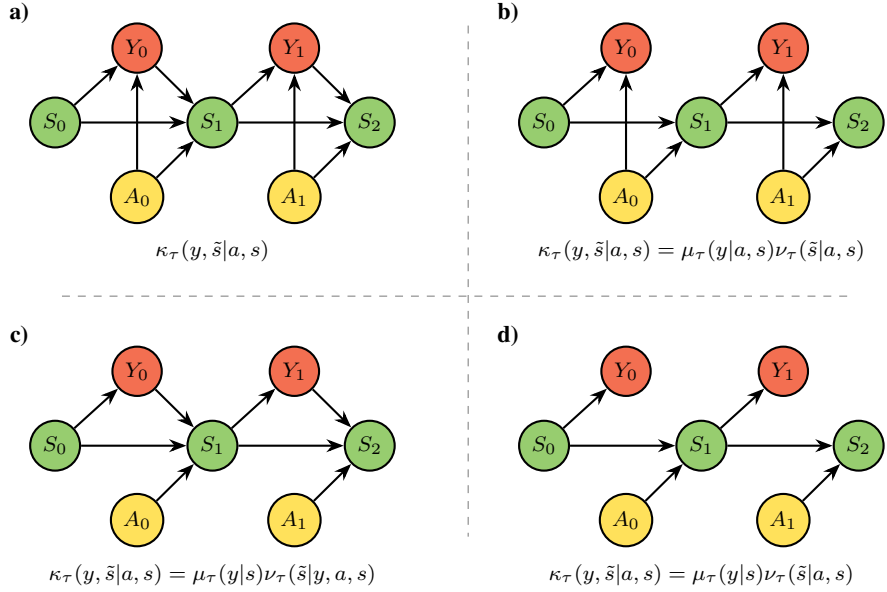
\begin{figure}[ht]
\centering
\begin{tikzpicture}[
    every node/.style={font=\footnotesize, minimum width=0.5em,thick, minimum size=0.3cm}, 
    state/.style={circle, draw, fill=YellowGreen},
    output/.style={circle, draw, fill=Red!70},
    input/.style={circle, draw, fill=Goldenrod!90},
    node distance=1.4cm,
    >=Stealth
    ]

\node (figA) at (0,0) { 
\begin{tikzpicture}
\node[state] (S0) {$S_0$};
\node[state, right=of S0] (S1) {$S_1$};
\node[state, right=of S1] (S2) {$S_2$};

\node[input,below right of=S0, xshift=0.1cm] (A0) {$A_0$};
\node[input,below right of=S1, xshift=0.1cm] (A1) {$A_1$};

\node[output,above right of=S0, xshift=0.1cm] (Y0) {$Y_0$};
\node[output,above right of=S1, xshift=0.1cm] (Y1) {$Y_1$};

\draw[->,line width=0.28mm] (S0) -- (Y0);
\draw[->,line width=0.28mm] (A0) -- (Y0);
\draw[->,line width=0.28mm] (A0) -- (S1);
\draw[->,line width=0.28mm] (S0) -- (S1);
\draw[->,line width=0.28mm] (Y0) -- (S1);

\draw[->,line width=0.28mm] (S1) -- (Y1);
\draw[->,line width=0.28mm] (A1) -- (Y1);
\draw[->,line width=0.28mm] (A1) -- (S2);
\draw[->,line width=0.28mm] (S1) -- (S2);
\draw[->,line width=0.28mm] (Y1) -- (S2);

\end{tikzpicture}
};
\node (A) at (figA.north west) {\small \textbf{a)}};
\node[below of=figA, yshift=-0.3cm] (kernelA) {$\kappa_\tau(y,\tilde{s}|a,s)$};

\node (figB) at (6.5,0) { 
\begin{tikzpicture}
\node[state] (S0) {$S_0$};
\node[state, right=of S0] (S1) {$S_1$};
\node[state, right=of S1] (S2) {$S_2$};

\node[input,below right of=S0, xshift=0.1cm] (A0) {$A_0$};
\node[input,below right of=S1, xshift=0.1cm] (A1) {$A_1$};

\node[output,above right of=S0, xshift=0.1cm] (Y0) {$Y_0$};
\node[output,above right of=S1, xshift=0.1cm] (Y1) {$Y_1$};

\draw[->,line width=0.28mm] (S0) -- (Y0);
\draw[->,line width=0.28mm] (A0) -- (Y0);
\draw[->,line width=0.28mm] (A0) -- (S1);
\draw[->,line width=0.28mm] (S0) -- (S1);

\draw[->,line width=0.28mm] (S1) -- (Y1);
\draw[->,line width=0.28mm] (A1) -- (Y1);
\draw[->,line width=0.28mm] (A1) -- (S2);
\draw[->,line width=0.28mm] (S1) -- (S2);

\end{tikzpicture}
};  
\node (B) at (figB.north west) {\small \textbf{b)}};
\node[below of=figB, yshift=-0.3cm] (kernelB) {$\kappa_\tau(y,\tilde{s}|a,s)=\mu_\tau(y|a,s)\nu_\tau(\tilde{s}|a,s)$};
\node (figC) at (0,-4.3) { 
\begin{tikzpicture}
\node[state] (S0) {$S_0$};
\node[state, right=of S0] (S1) {$S_1$};
\node[state, right=of S1] (S2) {$S_2$};

\node[input,below right of=S0, xshift=0.1cm] (A0) {$A_0$};
\node[input,below right of=S1, xshift=0.1cm] (A1) {$A_1$};

\node[output,above right of=S0, xshift=0.1cm] (Y0) {$Y_0$};
\node[output,above right of=S1, xshift=0.1cm] (Y1) {$Y_1$};

\draw[->,line width=0.28mm] (S0) -- (Y0);
\draw[->,line width=0.28mm] (A0) -- (S1);
\draw[->,line width=0.28mm] (S0) -- (S1);
\draw[->,line width=0.28mm] (Y0) -- (S1);

\draw[->,line width=0.28mm] (S1) -- (Y1);
\draw[->,line width=0.28mm] (A1) -- (S2);
\draw[->,line width=0.28mm] (S1) -- (S2);
\draw[->,line width=0.28mm] (Y1) -- (S2);

\end{tikzpicture}
};
\node (C) at (figC.north west) {\small \textbf{c)}};
\node[below of=figC, yshift=-0.3cm] (kernelC) {$\kappa_\tau(y,\tilde{s}|a,s)=\mu_\tau(y|s)\nu_\tau(\tilde{s}|y,a,s)$};

\node (figD) at (6.5,-4.3) { 
\begin{tikzpicture}
\node[state] (S0) {$S_0$};
\node[state, right=of S0] (S1) {$S_1$};
\node[state, right=of S1] (S2) {$S_2$};

\node[input,below right of=S0, xshift=0.1cm] (A0) {$A_0$};
\node[input,below right of=S1, xshift=0.1cm] (A1) {$A_1$};

\node[output,above right of=S0, xshift=0.1cm] (Y0) {$Y_0$};
\node[output,above right of=S1, xshift=0.1cm] (Y1) {$Y_1$};

\draw[->,line width=0.28mm] (S0) -- (Y0);
\draw[->,line width=0.28mm] (A0) -- (S1);
\draw[->,line width=0.28mm] (S0) -- (S1);

\draw[->,line width=0.28mm] (S1) -- (Y1);
\draw[->,line width=0.28mm] (A1) -- (S2);
\draw[->,line width=0.28mm] (S1) -- (S2);

\end{tikzpicture}
};  
\node (D) at (figD.north west) {\small \textbf{d)}};
\node[below of=figD, yshift=-0.3cm] (kernelD) {$\kappa_\tau(y,\tilde{s}|a,s)=\mu_\tau(y|s)\nu_\tau(\tilde{s}|a,s)$};

\draw[gray, dashed] (3.4,1.3) -- (3.4,-5.5);
\draw[gray, dashed] (-2,-2.3) -- (8.5,-2.3);

\end{tikzpicture}
\caption{Illustration of different types of transducers: Mealy transducers (a), output-Moore transducers (b), input-Moore transducer (c), and I-O Moore transducer.}
\label{fig:Mealy_vs_Moore} 
\end{figure}

\section{Proof of \autoref{lemma:funes}}
\label{app:proof_funes}

\begin{proof}
Let consider $S_t = \bm H_{t-1}$ and $S_0=0$. To prove that $S_t$ yields a transducer presentation of $\mathcal{I}(\bm Y|\bm A)$, we will use condition (1) from \autoref{lemma:transducer_faces}. 
For this, note that
\begin{align}\label{eq:00}
    p(s_{t+1},y_t | \bm s_{:t}, \bm h_{:t-1}, \bm a_{t:})
    = 
    p(s_{t+1},y_t | \bm h_{:t-1}, \bm a_{t:})
    =
    p(s_{t+1}| \bm h_{:t}, \bm a_{t+1:})
    p(y_t | \bm h_{:t-1}, \bm a_{t:}).
\end{align}
Let us develop each of those terms separately. First, one can find that
\begin{align}\label{eq:01}
    p(s_{t+1}|\bm h_{:t},\bm a_{t+1:})
    = p(s_{t+1}| s_{t}, h_{t},\bm a_{t+1:})
    = \delta_{s_{t+1}}^{(s_t,h_t)}
    = p(s_{t+1}| s_{t},h_{t}),
\end{align}
where $\delta_a^b$ is the Kroneker delta. Similarly, the second term can be developed as follows:
\begin{align}\label{eq:02}
    p(y_{t}| \bm h_{:t-1}, \bm a_{t:}) 
    = \frac{ p(\bm y_{:t}| \bm a_{:})}{ p(\bm y_{:t-1}|\bm a_{:})}
    \stackrel{(a)}{=} \frac{ p(\bm y_{:t}| \bm a_{:t})}{ p(\bm y_{:t-1}|\bm a_{:t})}
    = p(y_{t}| \bm y_{:t-1}, \bm a_{:t})
    = p(y_{t}| s_t, a_{t}),
\end{align}
where (a) uses the fact that the interface is anticipation-free. Finally, by combining~\autoref{eq:00}, ~\autoref{eq:01}, and ~\autoref{eq:02} one finds that
\begin{align}
    p(s_{t+1},y_t | \bm s_{:t}, \bm h_{:t-1}, \bm a_{t:})
    = 
    p(s_{t+1}| s_{t},h_{t})p(y_{t}| s_t, a_{t})
    = 
    p(s_{t+1},y_t| s_{t}, a_{t}),
\end{align}
which shows that condition (1) from \autoref{lemma:transducer_faces} is satisfied.
\end{proof}

\section{Proof of \autoref{lemma:homo_generate}}
\label{app:proof_homo_generate}

\begin{proof}
    Consider $S'_t = \phi(S_t)$ a reduction of the memory state $S_t$ of a transducer. Then
\begin{align}
    p(\bm y_{:t}\bm s'_{:t+1}|\bm a_{:}) 
    &= 
    \sum_{\tau=0}^t \!\!
    \sum_{\substack{ s_\tau\in\mathcal{S}\\ 
    \phi(s_\tau)=s'_\tau}} 
    \!\!
    p(\bm y_{:t}\bm s_{:t+1}|\bm a_{:}) \nonumber \\
    &\stackrel{(a)}{=} 
    \sum_{\tau=0}^t 
    \!
    \sum_{\substack{ s_\tau\in\mathcal{S}\\ 
    \phi(s_\tau)=s'_\tau}} 
    \!\!
    p(s_0) 
    \prod_{\tau=0}^{t} p(y_\tau|s_{\tau},a_\tau) p(s_{\tau+1}|s_{\tau},h_\tau) \nonumber \\
    &\stackrel{(b)}{=} 
    \sum_{\tau=0}^t 
    \!
    \sum_{\substack{ s_\tau\in\mathcal{S}\\ 
    \phi(s_\tau)=s'_\tau}} 
    \!\!
    p(s_0) \prod_{\tau=0}^{t} p(y_\tau|s'_{\tau},a_\tau) p(s_{\tau+1}|s_{\tau},h_\tau) \nonumber \\
    &=
    \sum_{\tau=0}^{t-1} 
    \!\!\!
    \sum_{\substack{ s_\tau\in\mathcal{S}\\ 
    \phi(s_\tau)=s'_\tau}} 
    \!\!\!\!
    p(s_0) 
    \prod_{\tau=0}^{t-1} p(y_\tau|s'_{\tau},a_\tau) p(s_{\tau+1}|s_{\tau},h_\tau) 
    p(y_t|s'_t,a_t)
    \!\!\!\!\!\!\!\!
    \sum_{\substack{s_{t+1}\in\mathcal{S}\\ \phi(s_{t+1})=s'_{t+1}}} 
    \!\!\!\!\!\!\!\!
    p(s_{t+1}|s_{t},h_\tau) 
    \nonumber \\
    &\stackrel{(c)}{=} 
    \sum_{\tau=0}^{t-1} 
    \!\!\!
    \sum_{\substack{ s_\tau\in\mathcal{S}\\ 
    \phi(s_\tau)=s'_\tau}} 
    \!\!\!\!
    p(s_0) 
    \prod_{\tau=0}^{t-1} p(y_\tau|s'_{\tau},a_\tau) p(s_{\tau+1}|s_{\tau},h_\tau) 
    p(y_t|s'_t,a_t)
    p(s'_{t+1}|s'_{t},h_\tau) 
    \nonumber \\
    &\stackrel{(d)}{=}  \ldots \nonumber \\
    &=
    \Bigg[\sum_{\substack{s_0\in\mathcal{S}\\\phi(s_0)=s'_0}} p(s_0) \Bigg]
    \prod_{\tau=0}^{t} p(y_\tau|s'_{\tau},a_\tau) 
    p(s'_{\tau+1}|s'_{\tau},h_\tau) \nonumber \\
    &\stackrel{(e)}{=} p(s'_0) 
    \prod_{\tau=0}^{t} p(y_\tau|s'_{\tau},a_\tau) 
    p(s'_{\tau+1}|s'_{\tau},h_\tau).\label{eq:end_proof}
\end{align}
Above, (a) uses that $S_t$ is the memory state of a transducer, (b) uses property (i) of homomorphisms (see~\autoref{def:minimal_world}), (c) and (e) uses property (ii) of homomorphisms, and (d) denotes that the same steps of previous equations are done iteratively. 
Finally, \autoref{eq:end_proof} together with \autoref{lemma:world_model} confirm that $S'_t$ yields a valid transducer for the same interface. 
\end{proof}

\section{Proof of~\autoref{lemma:bisim_homo}}
\label{app:proof_lemma_bisim_homo}

\begin{proof}
    Let's first assume that the mapping $\phi$ induces a reduction of the world model $S_t$ into $S'_t$, and define the equivalence relation $B$ such that $s\sim s'$ when $\phi(s)=\phi(s')$. In this setting, let's prove that $B$ is a bisimulation. For this, one can note that if $s\sim s'$ then one can use the first property of homomorphims to find that
    \begin{align}
        \Pr(Y_t=y|S_t=s, A_t=a) 
        &= 
        \Pr(Y_t=y|S'_t=\phi(s), A_t=a) \nonumber \\
        &= 
        \Pr(Y_t=y|S'_t=\phi(s'), A_t=a) \nonumber \\
        &=
        \Pr(Y_t=y|S_t=s', A_t=a).
    \end{align}
    Additionally, using the second property one finds that
    \begin{align}
        \sum_{s''\in[\tilde{s}]} 
        \Pr\big(S_{t+1}=s''| S_t=s,H_t=(y,a)\big)
        &=
        \Pr\big(S'_{t+1}=\tilde{s}| S'_t=\phi(s),H_t=(y,a)\big) \nonumber \\
        &=
        \Pr\big(S'_{t+1}=\tilde{s}| S'_t=\phi(s'),H_t=(y,a)\big) \nonumber \\
        &=
        \sum_{s''\in[\tilde{s}]} 
        \Pr\big(S_{t+1}=s''| S_t=s',H_t=(y,a)\big),
    \end{align}
    where $[\tilde{s}]=\{s\in\mathcal{S}:\phi(s)=\tilde{s}\}$.
    Together, these two results show that $B$ is a bisimulation in the sense of \autoref{def:bisimulation}.

    For proving the converse statement, let's assume that $B\subseteq \mathcal{S}\times\mathcal{S}$ is a bisimulation, and define $\phi(s)=[s]$ as a function that maps each state $s\in\mathcal{S}$ into its equivalence class according to $B$. 
    Let's prove that $S_t\xrightarrow{\phi} \phi(S_t) = [S_t]$ is a reduction. First, for $B$ being a bisimulation implies that $\Pr\big(Y_t=y| S_t=s, A_t=a\big) = \Pr\big(Y_t=y| S_t=s',A_t=a\big)$ for any $(s,s')\in B$, which in turn implies that
    \begin{align}
        \Pr\big(Y_t=y| \phi(S_t)=[s], A_t=a\big) = \Pr\big(Y_t=y| S_t=s,A_t=a\big),    
    \end{align}
    showing that $\phi$ satisfies the first property of homomorphisms. Furthermore, if $(s,s')\in B$ then 
\begin{align}
    \Pr\big(\phi(S_{t+1})=[\tilde{s}]| S_t=s,H_t=(y,a)\big) 
    &= \sum_{s''\in [\tilde{s}]}\Pr\big(S_{t+1}=s''| S_t=s,H_t=(y,a)\big) \nonumber \\
    &= \sum_{s''\in [\tilde{s}]}\Pr\big(S_{t+1}=s''| S_t=s',H_t=(y,a)\big) \nonumber \\
    &= \Pr\big(\phi(S_{t+1})=[\tilde{s}]| S_t=s',H_t=(y,a)\big) ,
\end{align}
which implies that 
\begin{equation}
    \Pr\big(\phi(S_{t+1})=[\tilde{s}]| S_t=s,H_t=(y,a)\big) = \Pr\big(\phi(S_{t+1})=[\tilde{s}]| \phi(S_t)=[s],H_t=(y,a)\big).
\end{equation}
Using this, one can finally show that
\begin{align}
    \Pr\big(\phi(S_{t+1})=[\tilde{s}]| \phi(S_t)=[s],H_t=(y,a)\big) 
    &= \sum_{s''\in [\tilde{s}]}\Pr\big(S_{t+1}=s''| \phi(S_t)=[s],H_t=(y,a)\big) \nonumber \\
    &= \sum_{s''\in [s]}\Pr\big(S_{t+1}=\tilde{s}| S_t=s,H_t=(y,a)\big)
\end{align}

\end{proof}

\section{Algorithms to reduce a transducer}
\label{app:algorithms_gtransducers}
For a given transducer with a finite number of possible memory states, one can reduce the corresponding world model as follows:
\begin{enumerate}
\item Compute a singular value decomposition $U_m = U \Lambda V^\intercal$, where $U\in\mathbb{R}^{m\times m}$ and $V\in\mathbb{R}^{n\times n}$ are unitary matrices of singular vectors and $\Lambda\in\mathcal{R}^{m\times n}$ is a diagonal matrix with $\text{Rank}(V_m)=r$ non-zero elements.
\item Collect the $r$ left singular vectors associated with non-zero singular values, and create the matrix $C = [\bm v_1,\dots, \bm v_n] \in\mathbb{R}^{n\times r}$.
\item Use $C$ as a transformation matrix to define the new world states, and calculate the resulting quasi-stochastic matrices.
\end{enumerate}
It can be shown that the resulting representation is minimal as in \autoref{def:minimal_world}. 
For more details on this procedure, see~\citep[Sec.~3]{balasubramanian1993equivalence} and also~\citep[Algorithm~1]{huang2015minimal}.

\section{Proof of \autoref{prop:predictive_beliefs}}
\label{app:beliefs_generate_interfaces}

\begin{proof}
To prove the first part of the proposition, let's consider the Bayesian beliefs of the memory state $S_t$ of a transducer $(S_t,Y_t,A_t)$ as given by $\dist{b}_{t}(s_{t})=p(s_{t}|\bm h_{:t-1})$. Let's also define the postdictive beliefs as $\dist{d}_{t}=p(s_{t}|\bm h_{:t})$. 
Then, their update can be calculated as follows:
\begin{align}\label{eq:predict1}
    \dist{b}_{t+1}(s_{t+1}) 
    = \sum_{s_{t}} p(s_{t+1}| s_{t},\bm h_{:t}) p(s_{t}|\bm h_{:t}) 
    \stackrel{(i)}{=} \sum_{s_{t}} p(s_{t+1}| s_{t},h_{t}) \dist{d}_{t}(s_{t}),
\end{align}
where (i) uses the Markovianity of the memory states following \autoref{lemma:transducer_as_world_model}. 
In a similar way, one can find that
\begin{align}\label{eq:update1}
    \dist{d}_{t}(s_t) 
    = \frac{p(s_{t}, \bm h_{:t-1},a_{t},y_{t})}{p(\bm h_{:t-1},a_{t},y_{t})}
    = \frac{p(y_{t}|s_{t}, \bm h_{:t-1},a_{t})p(s_{t}| \bm h_{:t-1},a_{t})}{p(y_{t}|\bm h_{:t-1},a_{t})}
    \stackrel{(ii)}{=} \frac{p(y_t|s_t,a_{t})}{Z'} \dist{b}_t(s_t)
\end{align}
with $Z'$ a normalising constant. Above, (ii) uses that $S_t$ is a world model together with the fact that that
\begin{align}
    p(s_{t}| \bm h_{:t-1},a_{t})
    = \frac{p(s_{t}, \bm h_{:t-1},a_{t})}{p(\bm h_{:t-1},a_{t})}
    = \frac{p(a_{t}|s_{t}, \bm h_{:t-1})p(s_{t}| \bm h_{:t-1})}{p(a_{t}|\bm h_{:t-1})}
    = p(s_{t}| \bm h_{:t-1}),
\end{align}
where the last equality holds due to the fact that actions depend on histories and not on states, and hence $p(a_{t}|s_{t},\bm h_{:t-1}) = p(a_{t}|\bm h_{:t-1})$. Then, combining \autoref{eq:predict1} and \autoref{eq:update1} one can find that
\begin{align}\label{eq:predictive_belief_update}
    \dist{b}_{t+1}(s_{t+1})
    =
    \hat{f}(\dist{b}_t,y_t,a_t)
    \coloneq
    \frac{1}{Z'}\sum_{s_{t}}  
    p(s_{t+1}| s_{t},h_{t}) p(y_{t}| s_{t},a_{t}) \dist{b}_{t}(s_{t}),
\end{align}
proving the first part of the proposition.

To prove the second part of the proposition, first note that 
\autoref{eq:predictive_belief_update} combined with condition (1) in \autoref{lemma:transducer_faces} imply that $(\dist{B}_t, A_t, Y_t)$ is a belief transducer with unifilar kernel given by
\begin{equation}
    \hat{\kappa}_\tau(y,\tilde{\dist{b}}|a,\dist{b})
    = 
    \delta_{\tilde{\dist{b}}}^{\hat{f}_\tau(\dist{b},y,a)} \sum_{\tilde{s},s\in\mathcal{S}}\kappa_\tau(y,\tilde{s}|a,s)\dist{b}(s)
    =
    \delta_{\tilde{\dist{b}}}^{\hat{f}_\tau(\dist{b},y,a)} 
    \Big( \bm 1^\intercal \cdot T^{(y|a)}_\tau \cdot \dist{b} \Big),
    \label{eq:xx1}
\end{equation}
where $T_\tau^{(y|a)}$ is the linear operator defined as in \autoref{eq:transducer_matrices}, and the second term corresponds to the probability of emitting $y$ given $\dist{b}$, i.e.
\begin{equation}\label{eq:emision}
    \bm 1^\intercal \cdot T^{(y_\tau|a_\tau)}_\tau \cdot \dist{b}_\tau = \sum_{s_\tau\in\mathcal{S}} p(y_\tau|s_\tau,a_\tau) \dist{b}_\tau(s_\tau) .
\end{equation}
Now, let's consider a transducer $(S_t, A_t, Y_t)$ and the belief transducer of predictive Bayesian beliefs $(\dist{B}_t, A_t, Y_t)$. 
Given that $(S_t, A_t, Y_t)$ is a transducer, then the update rule given by \autoref{eq:predictive_belief_update} can be re-written as 
\begin{equation}
    \hat{f}_\tau(\dist{b},y,a)
    =
    \frac{T^{(y|a)}_\tau \cdot \dist{b}}{\bm 1^\intercal \cdot T^{(y|a)}_\tau \cdot\dist{b}}.
\end{equation}
Moreover, for a given sequences $\bm y_{:\tau}, \bm a_{:\tau}$ and beliefs $\dist{b}_0,\dots,\dist{b}_\tau$ following this updating rule, this can be applied recursively yielding
\begin{align}
    \hat{f}_\tau(\dist{b}_\tau,y_\tau,a_\tau)
    &=
    \frac{T^{(y_\tau|a_\tau)}_\tau \cdot \dist{b}_\tau}{\bm 1^\intercal \cdot T^{(y_\tau|a_\tau)}_\tau \cdot\dist{b}_{\tau}} \nonumber \\
    &=
    \frac{T^{(\bm y_{\tau-1:\tau}|\bm a_{\tau-1:\tau})}_{\tau-1} \cdot \dist{b}_{\tau-1}}{\bm 1^\intercal \cdot T^{(\bm y_{\tau-1:\tau}|\bm a_{\tau-1:\tau})}_{\tau-1} \cdot\dist{b}_{\tau-1}} \nonumber \\
    &= \ldots \nonumber \\
    &=
    \frac{T^{(\bm y_{:\tau}|\bm a_{:\tau})}_{0} \cdot \dist{b}_{0}}{\bm 1^\intercal \cdot T^{(\bm y_{:\tau}|\bm a_{:\tau})}_{0} \cdot\dist{b}_{0}},
    \label{eq:xx2}
\end{align}
where we are using $T^{(\bm y_{t:t'}|\bm a_{t:t'})}_{t}\coloneq 
\prod_{\tau=t}^{t'} T^{(y_\tau|a_\tau)}_\tau$ as a shorthand notation. Note that, similarly as \autoref{eq:xx1}, the denominator in \autoref{eq:xx2} corresponds to
\begin{equation}
    1^\intercal \cdot T^{(\bm y_{:\tau}|\bm a_{:\tau})}_{0} \cdot\dist{b}_{0}
    =
    p(\bm y_{:\tau}|\bm a_{:\tau},s_0)\dist{b}_0(s_0).
    \label{eq:xx3}
\end{equation}
Finally, combining all these expressions one can directly calculate what is the result of successive applications of the kernel of a predictive Bayesian belief transducer:
\begin{align}
    \sum_{\dist{b}_{:t+1}} \prod_{\tau=0}^t \hat{\kappa}_\tau(y_\tau,\dist{b}_{\tau+1}|a_\tau,\dist{b}_\tau)
    \stackrel{(a)}{=}&
    \sum_{\dist{b}_{:t+1}} 
    \prod_{\tau=0}^t 
    \delta_{\dist{b}_{\tau+1}}^{\hat{f}(\dist{b}_\tau,y_\tau,a_\tau)} 
    \Big( \bm 1^\intercal \cdot  T^{(y_\tau|a_\tau)}_\tau \cdot \dist{b}_\tau \Big) \nonumber \\
    \stackrel{(b)}{=}&
    \prod_{\tau=0}^t 
    \bm 1^\intercal \cdot T^{(y_\tau|a_\tau)}_\tau \cdot 
    \frac{T^{(\bm y_{:\tau-1}|\bm a_{:\tau-1})}_{0} \cdot \dist{b}_{0}}{\bm 1^\intercal \cdot T^{(\bm y_{:\tau-1}|\bm a_{:\tau-1})}_{0} \cdot\dist{b}_{0}} \nonumber \\
    =&
    \bm 1^\intercal \cdot T^{(\bm y_{:t}|\bm a_{:t})}_{0} \cdot \dist{b}_{0} \nonumber \\
    \stackrel{(c)}{=}&
    \sum_{s_0} p(\bm y_{:\tau}|\bm a_{:\tau},s_0)\dist{b}_0(s_0).
\end{align}
Above, (a) uses~\autoref{eq:xx1}, (b) resolves the Dirac delta with the summation and uses~\autoref{eq:xx2}, and (c) uses~\autoref{eq:xx3}.  
To conclude, one can notice that if $\dist{b}_0(s_0 = \Pr(S_0=s_0)$, then this shows that $(S_t, A_t, Y_t)$ and $(\dist{B}_t, A_t, Y_t)$ generate the same interface.

\end{proof}

\section{Proof of \autoref{prop:postdictive_stuff}}
\label{app:postdictive_stuff}

Before the proof, let us note that the kernel of I-O Moore transducers can be re-organised as
\begin{align}
    p(\bm y_{:t},\bm s_{:t+1}|\bm a_{:t})
    &=
    p(s_0)
    \prod_{\tau=0}^t \kappa_\tau(y_{\tau},s_{\tau+1}|s_{\tau},a_{\tau}) \nonumber \\
    &=
    p(s_0)
    \prod_{\tau=0}^t \mu_\tau(y_{\tau}|s_\tau)\nu_\tau(s_{\tau+1}|s_\tau,a_{\tau}) \nonumber \\
    &=
    p(s_0)
    \prod_{\tau=0}^t \kappa^S_\tau(y_{\tau},s_{\tau}|s_{\tau-1},a_{\tau-1}),
\end{align}
where $\kappa^S_\tau(y_{t+1},s_{t+1}|s_t,a_{t})$ is a `time-shifted' kernel defined as
\begin{equation}
\kappa^S_\tau(y,\tilde{s}|s,a)
   \coloneq
\begin{cases}
    \mu_{\tau}(y|\tilde{s})\nu_{\tau-1}(\tilde{s}|s,a) \quad &\text{if }\tau\geq1, \nonumber \\
    \mu_0(y|s)\quad &\text{if }\tau=0.
\end{cases} 
\end{equation}
This means that I-O Moore transducers yield two associated kernels: the standard one $\kappa$ and the time-shifted one $\kappa^S$, and --- crucially --- both generate the interface. Following \autoref{eq:transducer_matrices}, let's defined the time-shifted linear operators
\begin{align}
\label{eq:transducer_matrices_shifted}
        T_\tau^{'(y|a)} \coloneq \sum_{i=1}^n\sum_{j=1}^n 
        \kappa^S_\tau(y,s_i|a,s_j) \bm e_{i} \bm e_j^\intercal~,
\end{align}
with the understanding that for $\tau=0$ then
\begin{align}
        T_0^{'(y|a)} \coloneq \sum_{i=1}^n\sum_{j=1}^n 
        \mu_0(y|s_j) \bm e_{i} \bm e_j^\intercal~.
\end{align}

\begin{proof}

To derive the update rule of postdictive beliefs, one can combine \autoref{eq:predict1} and \autoref{eq:update1} to find that
\begin{align}
    \dist{d}_{t+1}(s_{t+1}) 
    = 
    \frac{p(y_{t+1}|s_{t+1},a_{t+1})}{Z'} \sum_{s_{t}}p(s_{t+1}|s_t,h_{t})\dist{d}_t(s_t).
    \label{eq:postdictive_belief_update_proof}
\end{align}
Furthermore, if the transformer is I-O Moore, then 
\begin{align}
    p(y_{t+1}|s_{t+1},a_{t+1}) 
    = p(y_{t+1}|s_{t+1}) 
    \quad\text{and}\quad
    p(s_{t+1}|s_t,h_{t})
    =
    p(s_{t+1}|s_t,a_{t}).
\end{align}
Using these relationships in \autoref{eq:postdictive_belief_update_proof} directly yields the desired update rule.

To prove the second part of the proposition, let us first note for I-O Moore transducers then $p(y_{t+1}|s_{t+1}) = \mu_{t+1}(y_{t+1}|s_{t+1})$ and 
$p(s_{t+1}|s_t,a_{t})=\nu_t(s_{t+1}|s_t,a_{t})$. Using this, the update rule can be re-written as $\dist{d}_{\tau} = f'_\tau(y_{\tau},a_{\tau-1},\dist{d}_{\tau-1})$ with $f'$ given by
\begin{align}   
\Big[ f'_\tau(y,a,\dist{d})\Big](s)
\coloneq 
\frac{\mu_{\tau}(y|s)}{Z'} \sum_{s'} \nu_{\tau-1}(s|s',a)\dist{d}(s') 
=
\frac{1}{Z'}
\sum_{s'}\kappa_\tau^S(y,s|a,s')
\dist{d}(s').
\end{align}
Moreover, comparing this with \autoref{eq:predictive_belief_update}, one can find that
\begin{align}   
f'_\tau(y,a,\dist{d})
=
\frac{T^{'(y|a)}_\tau \cdot \dist{d}}{\bm 1^\intercal \cdot T^{'(y|a)}_\tau \cdot\dist{d}},
\end{align}
and following a similar derivation one finds that for sequences $\bm y_{:\tau}, \bm a_{:\tau}$ and beliefs $\dist{d}_0,\dots,\dist{d}_\tau$ %
\begin{align}
    f'_\tau(y_\tau,a_{\tau-1},\dist{d}_{\tau-1})
    &=
    \frac{T^{'(y_\tau|a_{\tau-1})}_\tau \cdot \dist{d}_{\tau-1}}{\bm 1^\intercal \cdot T^{'(y_\tau|a_{\tau-1})}_\tau \cdot\dist{d}_{\tau-1}} \nonumber \\
    &= \ldots \nonumber \\
    &=
    \frac{T^{(\bm y_{:\tau}|\bm a_{:\tau-1})}_{0} \cdot \dist{d}_{0}}{\bm 1^\intercal \cdot T^{(\bm y_{:\tau}|\bm a_{:\tau-1})}_{0} \cdot\dist{d}_{0}} ,
    \label{eq:yy2}
\end{align}
where we use $T^{'(\bm y_{t:t'}|\bm a_{t-1:t'-1})}_{t}\coloneq 
\prod_{\tau=t}^{t'} T^{'(y_\tau|a_{\tau-1})}_\tau$ as a shorthand notation. 

With these tools at hand, let's note that 
\autoref{eq:postdictive_belief_update} combined with condition (1) in \autoref{lemma:transducer_faces} imply that $(\dist{D}_t, A_t, Y_t)$ is a belief transducer with unifilar kernel given by
\begin{equation}
    \hat{\kappa}^S_\tau(y,\tilde{\dist{d}}|a,\dist{d})
    = 
    \delta_{\tilde{\dist{d}}}^{f'_\tau(y,a,\dist{d})} \sum_{s\in\mathcal{S}}
    \kappa_\tau^S(y,\tilde{s}|a,s)\dist{d}(s)
    =
\delta_{\tilde{\dist{d}}}^{f'_\tau(y,a,\dist{d})} 
    \Big( \bm 1^\intercal \cdot T^{'(y|a)}_\tau \cdot \dist{d} \Big),
    \label{eq:yy1}
\end{equation}
being analogous to the result found in \autoref{eq:xx1}. 
Then, combining all these expressions one can directly calculate what is the result of successive applications of the kernel of a predictive Bayesian belief transducer:
\begin{align}
    \sum_{\dist{d}_{:t}} \prod_{\tau=0}^t \hat{\kappa}^S_\tau(y_\tau,\dist{d}_{\tau}|a_{\tau-1},\dist{d}_{\tau-1})
    \stackrel{(a)}{=}&
    \sum_{\dist{d}_{:t}} 
    \prod_{\tau=0}^t 
    \delta_{\dist{d}_{\tau}}^{f'(y_\tau,a_{\tau-1},\dist{d}_{\tau-1})} 
    \Big( \bm 1^\intercal \cdot  T^{'(y_\tau|a_{\tau-1})}_\tau \cdot \dist{d}_{\tau-1} \Big) \nonumber \\
    \stackrel{(b)}{=}&
    \prod_{\tau=0}^t 
    \bm 1^\intercal \cdot T^{(y_\tau|a_{\tau-1})}_\tau \cdot 
    \frac{T^{(\bm y_{:\tau-1}|\bm a_{:\tau-2})}_{0} \cdot \dist{d}_{0}}{\bm 1^\intercal \cdot T^{(\bm y_{:\tau-1}|\bm a_{:\tau-2})}_{0} \cdot\dist{d}_{0}} \nonumber \\
    =&
    \bm 1^\intercal \cdot T^{(\bm y_{:t}|\bm a_{:t-1})}_{0} \cdot \dist{d}_{0} \nonumber \\
    =&
    \sum_{s_0}p(\bm y_{:\tau}|\bm a_{:\tau-1},s_0)\dist{d}_0(s_0).
\end{align}
Above, (a) uses~\autoref{eq:yy1}, (b) resolves the Dirac delta with the summation and uses~\autoref{eq:yy2}.  
To conclude, note that this shows that $(S_t, A_t, Y_t)$ and $(\dist{D}_t, A_t, Y_t)$ generate the same interface when $\dist{d}_0(s_0) = \Pr(S_0=s_0)$.

Before finishing, note that \autoref{cor:pred_up} follows directly from recognising \autoref{eq:predict1} and \autoref{eq:update1} as the equations related to `predict' and `update' steps in Kalman and Bayesian filtering~\citep{sarkka2023bayesian} corresponding to
\begin{equation}
    b_{t-1}=p(s_{t-1}|\bm h_{:t-1}) \xrightarrow{\text{predict}} d_{t}=p(s_{t}|\bm h_{:t-1}) \xrightarrow{\text{update}} b_{t}=p(s_{t}|\bm h_{:t}).
\end{equation}
\end{proof}

\section{Proof of~\autoref{prop:e_transducer_properties}}
\label{app:proof_e_transducer_properties}

This result was originally proven in~\cite[Prop.~2]{barnett2015computational}.  Here we provide an alternative proof that leverages the link established between transducer homomorphisms and bisimulations.

\begin{proof}
We first show that the equivalence class defined by \autoref{eq:epsilon} is a bisimulation of the world model $S_t=\bm H_{:t-1}$. 
For this, we first show that the coarse-graining defined by \autoref{eq:epsilon} is a bisimulation --- i.e., it it satisfies the two properties outlined in \autoref{def:bisimulation}. Condition (i) follows from \autoref{eq:epsilon} directly, since it only considers futures of length $L=1$. A proof that Condition (ii) follows from the fact that the dynamics of the equivalence classes are conditionally Markovian given the actions, which has been shown in \citep[Prop.~6]{barnett2015computational}. 

With this, the desired result follows directly from noting that $S_t=\bm H_{:t-1}$ is always a valid transducer of $\mathcal{I}(\bm Y|\bm A)$ (\autoref{lemma:funes}), and that the bisimulation of a transducer always yields a valid transducer (\autoref{lemma:bisim_homo}) that generates the same interface (\autoref{lemma:homo_generate}).
\end{proof}

\section{Proof of \autoref{theo:Predictive Transducer Coarse-Graining}}
\label{app:Predictive Transducer Coarse-Graining}

This proof is a direct extension of~\citep[Lemma~1]{barnett2015computational}, which focuses on `rival partitions' rather than predictive transducers. 
The core idea of the proof is that, for a given predictive transducer with memory $R_t\in\mathcal{R}$, one can build an equivalence relation in $\mathcal{R}\times\mathcal{R}$ induced by a coarse-graining map $\epsilon$ defined as
\begin{align}
    \epsilon'_t(r) & =\epsilon'_t(r')
    \quad\text{iff}\quad
    \Pr(\bm Y_{t:t'}|\bm A_{t:t'},R_t=r) 
    = 
    \Pr(\bm Y_{t:t'}|\bm A_{t:t'},R_t=r')
    \quad\forall t'\geq t.
\end{align}
Then, one can show that if $R_t$ is a predictive world model, then $\epsilon'_t(R_t)$ is isomorphic to the memory states of the $\epsilon$-transducer. The full proof of this is given next, after presenting the following lemma.

\begin{lemma}\label{lemma:cool_pred}
    A predictive transducer $(S_t, A_t, Y_t)$ satisfies
    \begin{align}
    p(\bm y_{t:t'}|\bm a_{t:t'},s_t)
    =
    p(\bm y_{t:t'}|\bm a_{t:t'},s_t,\bm h_{:t-1})
    =
    p(\bm y_{t:t'}|\bm a_{t:t'},\bm h_{:t-1})
    \quad\forall t'\geq t.
    \end{align}
\end{lemma}

\begin{proof}
By definition (see~\autoref{sec:worlds_of_beliefs}) a predictive transducer has memory states $S_t$ that satisfy the condition $I(\bm Y_{t:},S_t|\bm H_{:t-1}, \bm A_{t:})=0$, which implies that for all $t'\geq t$
\begin{align}
    p(\bm y_{t:t'}|\bm a_{t:t'},s_t,\bm h_{:t-1})
    &\stackrel{(i)}{=}
    p(\bm y_{t:t'}|\bm a_{t:},s_t,\bm h_{:t-1}) \nonumber \\
    &=
    p(\bm y_{t:t'}|\bm a_{t:},\bm h_{:t-1}) \nonumber \\
    &\stackrel{(ii)}{=}
    p(\bm y_{t:t'}|\bm a_{t:t'},\bm h_{:t-1}),
    \label{eq:001}
\end{align}
holding whenever $p(\bm y_{t:t'},s_t|\bm a_{t:t'},\bm h_{:t-1}) \neq 0$, which means that these events are compatible. Above, (i) and (ii) use the fact that the interface and world models are non anticipatory. 
Additionally, using the properties of transducers one can show that
\begin{align}
p(\bm y_{t:t'}|\bm a_{t:t'},s_t,\bm h_{:t-1})
&=
\sum_{\bm s_{t+1:t'+1}} p(\bm y_{t:t'},\bm s_{t+1:t'+1}|\bm a_{t:t'},s_t,\bm h_{:t-1}) \nonumber \\
&=
\sum_{\bm s_{t+1:t'+1}} 
\prod_{\tau=t}^{t'}
p(y_{\tau},s_{\tau+1}|\bm a_{t:t'},\bm s_{t:\tau},\bm y_{t:\tau-1},\bm h_{:t-1}) \nonumber \\
&\stackrel{(iii)}{=}
\sum_{\bm s_{t+1:t'+1}} 
\prod_{\tau=t}^{t'}
p(y_{\tau},s_{\tau+1}|\bm a_{t:t'},\bm s_{t:\tau},\bm y_{t:\tau-1}) \nonumber \\
&=
\sum_{\bm s_{t+1:t'+1}} p(\bm y_{t:t'},\bm s_{t+1:t'+1}|\bm a_{t:t'},s_t) \nonumber \\
&=
p(\bm y_{t:t'}|s_t,\bm a_{t:t'})
\label{eq:002}
\end{align}
for all $t'>t$, where (iii) uses condition (1) from \autoref{lemma:transducer_faces}. 
Finally, combining \autoref{eq:001} and \autoref{eq:002} gives the desired result.
\end{proof}

\begin{proof}[Proof of \autoref{theo:Predictive Transducer Coarse-Graining}]

Consider $S_t$ the memory state of a transducer. Let us first note that
\begin{align}\label{eq:convex}
    p(\bm y_{t:t'}|\bm a_{t:t'},\bm h_{:t-1})
    =
    \sum_{j\in\mathcal{J}} \alpha_j 
    p(\bm y_{t:t'}|\bm a_{t:t'},s_t,\bm h_{:t-1})
\end{align}
where $\alpha_j = p(s_t|\bm h_{:t-1},\bm a_{t:t'}) = p(s_t|\bm h_{:t-1})$, where the second equality holds because $S_t$ is a world model (property (2) in \autoref{def:world_model}). 
Above, we are using a suitable set of indices $\mathcal{J}$ corresponding to the possible values of $s_t$, which satisfy $\sum_{j\in\mathcal{J}}\alpha_j=\sum_{s_t\in\mathcal{S}}p(s_t|\bm h_{:t-1}) = 1$. 
Given that the Shannon entropy is concave, \autoref{eq:convex} implies 
\begin{equation}
    H\Big[p(\bm y_{t:t'}|\bm a_{t:t'},s_t)\Big]
    \geq
    \sum_{j\in\mathcal{J}} 
    \alpha_j H\Big[p(\bm y_{t:t'}|\bm a_{t:t'},s_t,\bm h_{:t-1})\Big],
    \label{eq:H_concave}
\end{equation}
where $H[p]$ is a shorthand notation for the entropy of a variable with distribution $p$. Moreover, given that $H$ is strictly concave, \autoref{eq:H_concave} turns into an equality if and only if all $\alpha_j$'s are either 0 or 1 for all $j\in\mathcal{J}$.

Now, consider $S_t$ the memory state of a predictive transducer and $M_t=\epsilon(\bm H_{:t-1})$ the memory state of the $\epsilon$-transducer, as determined by the coarse-graining mapping defined in~\autoref{eq:epsilon}. 
Note that $(M_t, A_t, Y_t)$ is a predictive transducer, and hence \autoref{lemma:cool_pred} can be used yielding
\begin{align}
p\big(\bm y_{t:t'}|\bm a_{t:t'},\epsilon(\bm h_{:t-1})\big)
=
p(\bm y_{t:t'}|\bm a_{t:t'},\epsilon(\bm h_{:t-1}),\bm h_{:t-1})
=
p(\bm y_{t:t'}|\bm a_{t:t'},\bm h_{:t-1})
\end{align}
for all $t'>t$. 
Then, using~\autoref{lemma:cool_pred} this time on $(S_t, A_t, Y_t)$, one can find that
\begin{align}
    \Pr\big(\bm Y_{t:t'}|\bm A_{t:t'},M_{t-1}=\epsilon(\bm h_{:t-1})\big)
    &=
    \Pr(\bm Y_{t:t'}|\bm A_{t:t'},\bm H_{:t-1}=\bm h_{:t-1})
     \nonumber \\
    &= 
    \Pr(\bm Y_{t:t'}|\bm A_{t:t'},S_t=s_t,\bm H_{:t-1}=\bm h_{:t-1}) 
    \label{eq:useful}
     \nonumber \\
    &=
    \Pr(\bm Y_{t:t'}|\bm A_{t:t'},S_t=s_t).
\end{align}
This implies that for each equivalence class $\epsilon_t(\bm h_{:t-1}) = [\bm h_{:t-1}]_{\sim_\epsilon}$ there exists at least one $s_t\in\mathcal{S}$ such that $\Pr\big(\bm Y_{t:t'}|\bm A_{t:t'},M_{t-1}=\epsilon(\bm h_{:t-1})\big)
=
\Pr(\bm Y_{t:t'}|\bm A_{t:t'},S_t=s_t)$. 
Moreover, the previous equalities imply that if $S_t$ is a predictive transducer, then \autoref{eq:H_concave} necessarily becomes an equality. This, in turn, implies that $\alpha_j=p(s_t|\bm h_{:t-1})$ is 1 for all $s_t\in\mathcal{S}$ for which 
\begin{align}
\Pr(\bm Y_{t:t'}|\bm A_{t:t'},S_t=s_t)
&= 
\Pr(\bm Y_{t:t'}|\bm A_{t:t'},S_t=s_t,\bm H_{:t-1}=\bm h_{:t-1}) \nonumber \\
&=
\Pr(\bm Y_{t:t'}|\bm A_{t:t'},\bm H_{:t-1}=\bm h_{:t-1})
\end{align}
holds, or 0 otherwise. 
This implies that the mapping $\epsilon'$ given by
\begin{align}
\epsilon'(s_t) = \epsilon'(s_t')
\quad\Leftrightarrow\quad
\Pr(\bm Y_{t:t'}|\bm A_{t:t'},S_t=s_t) 
=
\Pr(\bm Y_{t:t'}|\bm A_{t:t'},S_t=s_t')
\quad \forall t'\geq t
\end{align}
satisfies $\epsilon'_t(S_t) = \epsilon_t(\bm H_{:t-1})= M_t$.
\end{proof}

\section{Comparing the reduction of general vs predictive transducers}
\label{app:general_vs_predictive}

Building upon the discussion about the canonical dimension of a transducer (see \autoref{eq:cannonical_dimension}), let us focus on transducers with finite memory states (i.e. $|\mathcal{S}|=n$) and consider the matrix $W$ whose columns given by the vectors $\bm w(\bm h_{:t})\in\mathbb{R}^n$ of probabilities of generating $\bm y_{:t}$ given $\bm a_{:t}$ when starting from different world states, so that its $k$-th coordinate is 
$[\bm w(\bm h_{:t})]_k = \Pr(\bm Y_{:t} = \bm y_{:t} | \bm A_{:t}=\bm a_{:t}, S_0 = s_k)$ for all possible sequences when $t=n-1$ (see~\citep[Prop.~4.3]{cakir2021theory}).  
Then, the coarse-graining $\epsilon$ defined by~\autoref{eq:epsilon} correspond to merging together all rows of $W_t$ that are equal. 
In contrast, the canonical dimension $d(\mathcal{T})$ defined in~\autoref{eq:cannonical_dimension} corresponds to the number of linearly independent rows. 
The crucial point is that, if a transducer with memory states $S_t$ is predictive, then any coarse-graining $\epsilon(S_t)$ will also be predictive. However, reductions via more general procedures to trim linearly dependent components may not be attainable via coarse-grainings. In particular, the matrix $W_t$ of an $\epsilon$-transducer may have linearly dependent rows, and reducing those would --- due to \autoref{corollary:e_transducer} --- necessary make the transducer to stop being predictive.

It is interesting to note that the predictive beliefs of the $\epsilon$-transducer are isomorphic to the states of the $\epsilon$-transducer. 
However, the $\epsilon$-transducer is not the only machine whose MSP produces the $\epsilon$-transducer --- in fact, the MSP of any transducer without redundant states will produce the same.

Finally, it is worth noting that a non-predictive finite world model may have an associate $\epsilon$-transducer that has infinite states filling the simplex of belief states with intricate fractal patterns~\citep{jurgens2021divergent}. This fact makes techniques to study transducers at various degrees of resolution (e.g., via approximate homomorphisms~\citep{taylor2008bounding} or bisimulation~\citep{girard2011approximate}, or using rate-distortion trade-offs~\citep{marzen2016predictive}) particularly important.

\section{A canonical retrodictive world model}
\label{app:profet_tree}

One can show that all anticipation-free transducer have one retrodictive transducer that is somehow dual to \textit{Funes the memorious} (\autoref{lemma:funes}), and can be described as \textit{Sunef the prophet}, with knowledge of all possible sequence of future actions.
To build the world model of this transducer, let us first denote as $\mathcal{T}_\mathcal{A}$ the regular tree with one root and where each node has one branch per element in $\mathcal{A}$. Let us denote by $\mathcal{N}(\mathcal{T}_\mathcal{A})$ the nodes of the tree, and establish some operations:
\begin{itemize}
    \item $\mu: \mathcal{N}(\mathcal{T}_\mathcal{A})\to \mathcal{A}^*$ and $\nu: \mathcal{N}(\mathcal{T}_\mathcal{A})\to \mathcal{N}(\mathcal{T}_\mathcal{A})^*$ with $()^*$ the Kleene operator, where $\mu(v)$ and  $\nu(v)$ returns a vector with all the branches and nodes in the path leading back from $v$ to the root, respectively.
    \item $\pi:\mathcal{N}(\mathcal{T}_\mathcal{A})\times \mathcal{A}\to  \mathcal{N}(\mathcal{T}_\mathcal{A})$, where $\pi(v,a)$ gives the descendent of $v$ connected via branch $a$. 
   \item $\tau:\mathcal{N}(\mathcal{T}_\mathcal{A})\to\mathbb{N}$, where $\tau(v)$ is the depth of $v$ in the tree.
\end{itemize}
With all this, we are ready to define the world model. 
In general, $S_t \in \mathcal{Y}^{\mathcal{T}_\mathcal{A}}$ are random variables that take values on $\mathcal{T}_\mathcal{A}$-shaped sequences of symbols in $\mathcal{Y}$. 
Concretely, $S_0 = \big(Z_v: v\in\mathcal{N}(\mathcal{T}_\mathcal{A})\big)$ with $Z_v\in\mathcal{Y}$ being random variables, whose joint distribution is given by
\begin{equation}
    \Pr\big(S_0 = (Z_v:v\in\mathcal{T}_{\mathcal{A}}) \big) := \prod_{v\in{\mathcal{T}_\mathcal{A}}} \Pr(Z_v|\bm Z_{\nu(v)})
\end{equation}
with $\bm Z_{\nu(v)}$ the vector of variables corresponding to nodes in $\nu(v)$ and 
\begin{equation}
\Pr(Z_v = y |\bm Z_{\nu(v)} = \bm y_{:\nu(v)-1}) 
:= \Pr\big(Y_{\tau(v)} = y | \bm Y_{:\tau(v)-1} = \bm y_{:\nu(v)-1},\bm A_{:\tau(v)}=\mu(v)\big)     
\end{equation}
Then, the world's dynamics are established recursively by 
$p(s_{t+1}|\bm s_{:t},\bm h_:) := \delta_{s_{t+1}}^{f(s_t,a_t)}$ 
so that $S_{t+1} = f(S_t,A_t)$ a.s., with the unifilar update established by 
\begin{equation}
    S_{t+1} = (Z_v^{t+1}: v\in\mathcal{T}_\mathcal{A})
    \quad
    \text{with}
    \quad 
    Z_v^{t+1} = Z_{\pi(v,A_t)}^t. 
\end{equation}

In summary, the world is first initialised at time zero by sampling $S_0$, i.e. by sampling $Z_v$ for all $v\in\mathcal{T}_\mathcal{A}$ --- which stands to sample $\bm Y_:$ for all possible sequences of actions $\bm a_:$. After this, the world evolves deterministically by following the update rule given by $f$.

\section{Example of a non-reversible transducer}
\label{app:non-reversible_transducers}

Let us provide an example of how reversing a transducer can lead to a violation of the anticipation free condition. 
Let's consider the so-called delay channel, for which the output $Y_{t+1}$ is equal to the previous action $A_t$~\citep{barnett2015computational}. This channel displays acausal behaviour when time reversed: somehow the action $A_t$ determines the outcome at the previous time step $Y_{t-1}$, meaning that
\begin{align}
    I(\bm Y_{:t-1};\bm A_{t:}|\bm A_{t-1:}) 
    =
    I(Y_{t-1};A_{t}|\bm A_{t-1:})
    = 
    H(A_{t}|\bm A_{t-1:}),
\end{align}
which is nonzero if the entropy rate of the actions is nonzero. This leads to an interface that does not satisfy anticipation free, violating the conditions of \autoref{def:interface}.

\section{Reversing processes and proof of \autoref{teo_reversing}}
\label{app: Reversing Processes}

Here we present an extended exposition of the conditions for reversing various types of stochastic processes.

\subsection{Reversing Markov processes}

Let's start by considering a Markov process $X_t$, so that $p(x_t|\bm x_{:t-1}) = p(x_t|x_{t-1})$ for all $t\in\mathbb{N}$. Then the reverse process (given by $X_t,X_{t-1},\dots$) is also Markov, as
\begin{align}
    p(x_t|\bm x_{t+1:t'}) 
    &= 
    \frac{ p(\bm x_{t:t'}) }{ p(\bm x_{t+1:t'}) }
    = 
    \frac{ p(x_t) \prod_{k=t+1}^{t'} p(x_k|\bm x_{t:k-1}) }
    { p(x_{t+1}) \prod_{j=t+2}^{t'} p(x_j|\bm x_{t+1:j-1}) } \nonumber \\
    &=  
    \frac{ p(x_t) \prod_{k=t+1}^{t'} p(x_k|x_{k-1}) }
    { p(x_{t+1}) \prod_{j=t+2}^{t'} p(x_j|x_{j-1}) } 
    =
    \frac{ p(x_t) p(x_{t+1}| x_{t}) }
    { p(x_{t+1}) }
    = p(x_t|x_{t+1}).
\end{align}

\subsection{Reversing hidden Markov models}
\label{app:hmm_reverse}

Let's now consider a general (i.e. Mealy~\citep{riechers2016exact}) hidden Markov model, in which $p(s_{t+1},y_t|\bm s_{:t}, \bm y_{:t-1}) = p(s_{t+1},y_t|s_t)$ holds for all $t \in \mathbb{N}$. 
As for Markov chains, one can show that the reverse process is also an hidden Markov model, as
\begin{align}
    p(s_t,y_t|\bm s_{t+1:t'+1}, \bm y_{t+1:t'})
    &= 
    \frac{p(\bm s_{t:t'+1}, \bm y_{t:t'})}{p(\bm s_{t+1:t'+1}, \bm y_{t+1:t'})} \nonumber \\
    &= 
    \frac{ p(s_t,y_t,s_{t+1}) \prod_{k=t+1}^{t'} p(s_{k+1},y_k|\bm s_{t:k},\bm y_{t:k-1}) }
    { p(s_{t+1},y_{t+1},s_{t+2}) \prod_{j=t+2}^{t'} p(s_{j+1},y_j|\bm s_{t:j},\bm y_{t:j-1}) } \nonumber \\
    &= 
    \frac{ p(s_t,y_t,s_{t+1}) \prod_{k=t+1}^{t'} p(s_{k+1},y_k|s_{k}) }
    { p(s_{t+1},y_{t+1},s_{t+2}) \prod_{j=t+2}^{t'} p(s_{j+1},y_j|s_{j}) } \nonumber \\
    &= 
    \frac{ p(s_t) \prod_{k=t}^{t'} p(s_{k+1},y_k|s_{k}) }
    { p(s_{t+1}) \prod_{j=t+1}^{t'} p(s_{j+1},y_j|s_{j}) } \nonumber \\
    &= 
    \frac{ p(s_t) p(s_{t+1},y_{t}|s_{t}) }
    { p(s_{t+1}) } \nonumber \\
    &= 
    p(s_{t},y_{t}|s_{t+1}).
\end{align}
Note that in this case the structure is not perfectly time-symmetric, but could be described as `co-Mealy' structure --- as the time indices of the world are shifted. 

If the hidden Markov model is Moore~\citep{riechers2016exact}, so that $p(s_{t+1},y_t|\bm s_{:t}, \bm y_{:t-1}) = p(s_{t+1}|s_t)p(y_{t}|s_t)$, then a similar calculation leads to
\begin{equation}
    p(s_t,y_t|\bm s_{t+1:t'+1}, \bm y_{t+1:t'})
    = \frac{ p(s_t) p(s_{t+1},y_{t}|s_{t}) }
    { p(s_{t+1}) }
    = \frac{ p(s_t) p(s_{t+1}|s_{t}) p(y_{t}|s_{t}) }
    { p(s_{t+1}) }
    = p(s_{t}|s_{t+1}) p(y_{t}|s_{t}),
\end{equation}
yielding another Moore hidden Markov model.

\subsection{Reversing transducers}

Using the previous calculations as a foundation, let's now explore the reverse properties of a transducer, where $p(s_{t+1},y_t|\bm s_{:t}, \bm y_{:t-1}, \bm a_:) = p(s_{t+1},y_t|s_t,a_t)$ holds (see~\autoref{app:proof_transducer_faces}). 
Using this property, it is direct to see that
\begin{align}
    p(\bm y_{:t},\bm s_{:t+1}| \bm a_{:}) 
    &= 
    p(s_0)\prod_{\tau=0}^t p(y_\tau,s_{\tau+1}|\bm y_{:\tau-1},\bm s_{:\tau},\bm a_{:}) \nonumber \\
    &=
    p(s_0)\prod_{\tau=0}^t p(y_\tau,s_{\tau+1}|s_{\tau},\bm a_{:t}) \nonumber \\
    &=
    p(\bm y_{:t},\bm s_{:t+1}| \bm a_{:t}) ,
\end{align}
showing that transducers naturally impose some arrow of time over actions.\footnote{Note that the derivation uses the fact that $p(s_0|\bm a_:) = p(s_0)$, and it wouldn't work for other initial point where this does not hold.} 
Now, let's consider expressing $p(\bm y_{:t},\bm s_{:t+1}| \bm a_{:})$ factor backwards as follows:
\begin{align}
    p(\bm y_{:t},\bm s_{:t+1}| \bm a_{:}) 
    = p(\bm y_{:t},\bm s_{:t+1}| \bm a_{:t})
    = p(s_{t+1}|\bm a_{:t})\prod_{\tau=0}^t p(y_\tau,s_\tau|\bm y_{\tau+1:t},\bm s_{\tau+1:t+1},\bm a_{:t}).
\end{align}
This shows that we need to looks for ways of simplifying expressions of the form $p(y_\tau,s_\tau|\bm y_{\tau+1:t},\bm s_{\tau+1:t+1},\bm a_{:t})$. Using the properties of transducers, one can show that
\begin{align}
    p(s_\tau,y_\tau|\bm s_{\tau+1:t+1}, \bm y_{\tau+1:t},\bm a_{:t})
    &= 
    \frac{p(\bm s_{\tau:t+1}, \bm y_{\tau:t},\bm a_{:t})}{p(\bm s_{\tau+1:t+1}, \bm y_{\tau+1:t},\bm a_{:t})} \nonumber \\
    &= 
    \frac{ p(s_\tau,y_\tau,s_{\tau+1},\bm a_{:t}) \prod_{k=\tau+1}^{t} p(s_{k+1},y_k|\bm s_{\tau:k},\bm y_{\tau:k-1},\bm a_{:t}) }
    { p(s_{\tau+1},y_{\tau+1},s_{\tau+2},\bm a_{:t}) \prod_{j=\tau+2}^{t} p(s_{j+1},y_j|\bm s_{\tau:j},\bm y_{\tau:j-1},\bm a_{:t}) } \nonumber\\
    &= 
    \frac{ p(s_\tau,y_\tau,s_{\tau+1},\bm a_{:t}) \prod_{k=\tau+1}^{t} p(s_{k+1},y_k|s_{k},a_k) }
    { p(s_{\tau+1},y_{\tau+1},s_{\tau+2},\bm a_{:t}) \prod_{j=\tau+2}^{t} p(s_{j+1},y_j|s_{j},a_j) } \nonumber \\
    &= 
    \frac{ p(s_\tau,\bm a_{:t}) \prod_{k=\tau}^{t} p(s_{k+1},y_k|s_{k}, a_k) }
    { p(s_{\tau+1},\bm a_{:t}) \prod_{j=\tau+1}^{t} p(s_{j+1},y_j|s_{j}, a_j) } \nonumber \\
    &= 
    \frac{ p(s_\tau|\bm a_{:t}) p(s_{\tau+1},y_{\tau}|s_{\tau},a_\tau) }
    { p(s_{\tau+1}|\bm a_{:t}) } \nonumber \\
    &= 
    \frac{ p(s_\tau|\bm a_{:t}) p(s_{\tau+1}|s_{\tau},a_\tau)}
    { p(s_{\tau+1}|\bm a_{:t}) }
    p(y_{\tau}|s_{\tau+1},s_{\tau},a_\tau) \nonumber \\
    &= 
    p(s_{\tau}|s_{\tau+1},\bm a_{:t})
    p(y_{\tau}|s_{\tau},s_{\tau+1}, a_{\tau}).
\end{align}
This shows that, for any transducer $(S_t, A_t, Y_t)$, we can always `run it back' using the whole sequence of actions to reproduce the interface, as shown by the factorisation
\begin{align}
    p(\bm y_{:t},\bm s_{:t+1}| \bm a_{:}) 
    = p(s_{t+1}|\bm a_{:t}) 
    \prod_{\tau=0}^t 
    p(s_\tau|s_{\tau+1},\bm a_{:t})
    p(y_\tau|,s_\tau,s_{\tau+1}, a_{\tau}).
\end{align}
If the transducer satisfies the additional condition
$p(s_\tau|s_{\tau+1},\bm a_{:t}) 
= p(s_\tau|s_{\tau+1}, a_\tau)$,
or equivalently the information relation
    $I(\bm S_\tau;\bm A_{0:\tau-1}\bm A_{\tau+1:t}|S_{\tau+1}, A_\tau)=0$, 
then one obtains a reverse factorisation of the form
\begin{align}
\label{eq:reversed_factorisation}
    p(\bm y_{:t},\bm s_{:t+1}| \bm a_{:}) 
    = p(s_{t+1}|\bm a_{:t})\prod_{\tau=0}^t p(y_\tau,s_\tau|s_{\tau+1},a_t).
\end{align}
The reverse kernel $\kappa^R$ can be expressed in terms of the forward one via the following derivation:
\begin{align}
    \kappa^R_\tau(y_\tau,s_\tau|a_\tau,s_{\tau+1})
    &=
    p(y_\tau,s_\tau|a_{:\tau},s_{\tau+1})
    =
    \frac{p(s_{\tau}|a_{\tau})}{p(s_{\tau+1}|a_{\tau})} p(y_\tau,s_{\tau+1}|a_{:\tau},s_{\tau}),
\end{align}
which implies that for reversible transducers then
\begin{equation}
\label{eq:reversed_kernel0}
    \kappa^R_\tau(y,s|a,\tilde{s}) 
    = 
    \frac{\Pr(S_\tau=s|A_\tau=a)}{\Pr(S_{\tau+1}=\tilde{s}|A_\tau=a)}
    \kappa_\tau(y,\tilde{s}|a,s).
\end{equation}

It is interesting to note that that replacing \autoref{eq:reversed_kernel0} in \autoref{eq:reversed_factorisation} could give the impression of not recovering \autoref{eq:transducer}; however, under the assumption of transducer reversibility it does. 
To confirm this, let us first check that 
\begin{align}
    p(s_\tau|s_{\tau+1},\bm a_{:t})
    =
    \frac{p(s_\tau,s_{\tau+1}|\bm a_{:t})}{p(s_{\tau+1}|\bm a_{:t})}
    =
    \frac{p(s_\tau|\bm a_{:t}) p(s_{\tau+1}|s_\tau,\bm a_{:t})}{p(s_{\tau+1}|\bm a_{:t})},
\end{align}
and similarly
\begin{align}
    p(s_\tau|s_{\tau+1},a_{\tau})
    =
    \frac{p(s_\tau,s_{\tau+1}|a_{\tau})}{p(s_{\tau+1}|a_{\tau})}
    =
    \frac{p(s_\tau|a_{\tau}) p(s_{\tau+1}|s_\tau,a_{\tau})}{p(s_{\tau+1}|a_{\tau})}.
\end{align}
Now, as $s_\tau$ is the memory state of a transducer then $p(s_{\tau+1}|s_\tau,\bm a_{:t}) = p(s_{\tau+1}|s_\tau,a_{t})$. Putting all together, the condition $p(s_t|s_{t+1},\bm a_{:t})=p(s_t|s_{t+1},a_{t})$ can be seen to imply the following:
\begin{equation}
      \frac{p(s_\tau|\bm a_{:t})}{p(s_{\tau+1}|\bm a_{:t})}
      =
      \frac{p(s_\tau|a_{\tau})}{p(s_{\tau+1}|a_{\tau})}.
\end{equation}
This equality allows us to confirm that replacing 
\autoref{eq:reversed_kernel} in \autoref{eq:reversed_factorisation} indeed recovers \autoref{eq:transducer}.

In conclusion, if $p(s_\tau|s_{\tau+1},\bm a_{:t}) 
= p(s_\tau|s_{\tau+1}, a_\tau)$ holds then one can generate the interface via the following procedure:
\begin{enumerate}
    \item Initialise the state of the world at time $t+1$ by sampling $p(s_{t+1}|\bm a_{:t})$. Alternatively, for counterfactual analyses pick an arbitrary world state $s\in\mathcal{S}$ and set $S_{t+1}=s$ .
    \item Run the transducer backwards using the kernel $\kappa^R(y_\tau,s_\tau|a_\tau,s_{\tau+1})=p(y_\tau,s_\tau|s_{\tau+1},a_t)$.
\end{enumerate}

\subsection{Effect of action-unifiliarity}

A transducer is action-unifilar if $p(s_{\tau+1}|s_\tau,a_\tau) = \delta_{s_{\tau+1}}^{f(s_\tau,a_\tau)}$ with $S_{\tau+1} = f(S_\tau,A_\tau)$ a function. 
If the dynamics of the transducer is action-counifilar, meaning that $p(s_{\tau}|s_{\tau+1},a_\tau)=\delta_{s_\tau}^{r(s_{\tau+1},a_\tau)}$ where $S_{\tau}=r(S_{\tau+1},A_\tau)$, then we necessarily safisfy the condition of being reversible $p(s_\tau|s_{\tau+1},a_{:\tau})=p(s_\tau|s_{\tau+1},a_{\tau})$.  However, this is much more restrictive if than action-unifilarity if we insist that every world-state can accept every action $\sum_{s_{\tau+1}}p(s_{\tau+1}|s_\tau,a_\tau)=1$.  Using Bayes rule
\begin{align}
    p(s_{\tau+1}|s_\tau,a_\tau)&=p(s_\tau|s_{\tau+1},a_\tau) \frac{p(s_{\tau+1}|a_\tau)}{p(s_{\tau}|a_\tau)} \nonumber \\ 
    &=\delta_{s_\tau}^{r(s_{\tau+1},a_\tau)} \frac{p(s_{\tau+1}|a_\tau)}{p(s_{\tau}|a_\tau)},
\end{align}
we see that there is one nonzero transition for every combination of state $s_{\tau+1}$ and action $a_\tau$.  We can think of each transition as an edge betweens states labeled with the action, like a driven transition.  This means that there are $|\mathcal{A}|$ transitions per state $s_\tau$.  The condition that every world-state can accept every action means that every state has at least one outgoing edge for every action.  If this were a non-unifilar model, this would mean that there an action that had two or more outgoing edges.  However, that would mean that the total number of edges in the automata is larger than $|\mathcal{A}||\mathcal{S}|$, which is a contradiction.  Thus, each state $s_\tau$ has exactly one outgoing edge for each action $a_\tau$, meaning that the next state is a function of these states
\begin{align}
S_{\tau+1}=f(S_\tau,A_\tau).
\end{align}
Therefore, every action-counifilar transducer is also action-unifilar, meaning that it obeys a type of reversibility.

\section{Proof of \autoref{teo:retro_MSP}}
\label{app:proof_retro_mix}

For convenience, in this proof we use Dirac's notation, which 
uses bras like $\langle v|$ and kets like $|v\rangle$ to express row and column vectors respectively.  If we are describing vectors and matrices over states $\mathcal{S}$, then we can use an orthonormal basis ($\{ | s \rangle \}_{s \in \mathcal{S}}$ such that $\langle s |s' \rangle =\delta_{s,s'}$) in the Hilbert space $\mathcal{H}_\mathcal{S}$ to express the vector
\begin{align}
    |v \rangle = \sum_{s} v(s) |s \rangle.
\end{align}
Here, $v(s)$ represents the $s$th element of the vector.  Similarly, for a linear operator in this Hilbert space, we can think of 
\begin{align}
\langle s' |M |s \rangle,
\end{align}
as the element in the $s$th row and $s'$th column, and we can translate a matrix $A$ with elements $A_{ss'}$ into a linear operator in this space by using the outer-product
\begin{align}
A=\sum_{ss'}|s' \rangle A_{ss'}\langle s|.
\end{align}

Using this notation, we can consider a vector space $\mathbb{R}^{|\mathcal{S}|}$ using the orthonormal basis of states $\{|s \rangle \}_{s \in \mathcal{S}}$ such that $\langle s | s' \rangle = \delta_{s,s'}$. 
Then, the predictive Bayesian belief can be described as
\begin{align}
\label{eq:pred_quantum}
    |\rho^P(\bm y_{:t}, \bm a_{:t}) \rangle = \sum_{s_{t+1}} |s_{t+1} \rangle p(s_{t+1}|\bm h_{:t})~,
\end{align}
and the retrodictive Bayesian belief as
\begin{align}
\label{eq:retrodic_quantum}
    \langle \rho^R(\bm y_{:t}, \bm a_{:t}) | = \sum_{s_0}  p(s_{0}|\bm h_{:t})\langle s_0|~.
\end{align}
Similarly, the matrix corresponding a sequence of actions $\bm a_{:t}$ and outputs $\bm y_{:t}$ can be described as
\begin{align}
T^{(\bm y_{:t}|\bm a_{:t})} 
=
\prod_{\tau=0}^{t} T^{(y_{\tau}|a_{\tau})} 
= \sum_{s_0,s_{\tau+1}}|s_{\tau+1} \rangle p(s_{\tau+1} ,\bm y_{:\tau}|\bm a_{:\tau},s_0) \langle s_0 |,
\end{align}
If we define the initial diagonal operator $\rho_t\equiv \sum_{s_t} |s_t \rangle p(s_t) \langle s_t |$, 
then we can calculate the probability of joint start and end state as follows
\begin{align}
\label{eq:matrix_prod}
T^{(\bm y_{:\tau}|\bm a_{:\tau})} \rho_0
= 
\sum_{s_0,s_{\tau+1}}|s_{\tau+1} \rangle p(s_{\tau+1},s_0 ,\bm y_{:\tau}|\bm a_{:\tau}) \langle s_0 |.
\end{align}
With this, one can calculate the interface via expressions of the form
\begin{align}
\label{eq:inter_quantum}
p(\bm y_{:\tau}|\bm a_{:\tau})
= 
\langle 1 | T^{(\bm y_{:\tau}|\bm a_{:\tau})}  \rho_0 |1 \rangle,
\end{align}
where $|1 \rangle \equiv \sum_{s} |s \rangle.$

\begin{proof}[Proof of \autoref{teo:retro_MSP}]
    
Using this notation, the BMSM (\autoref{sec:retro_beliefs}) can be expressed as
\begin{align}
\rho(\bm y_{:\tau},\bm a_{:\tau})
=
\sum_{s_0,s_{\tau+1}}|s_{\tau+1} \rangle p(s_{\tau+1},s_0 |\bm y_{:\tau},\bm a_{:\tau}) \langle s_0 | ~.
\end{align}
Moreover, by using \autoref{eq:matrix_prod} and \autoref{eq:inter_quantum} one can find that
\begin{align}
    \rho(\bm y_{:\tau},\bm a_{:\tau})
    =
    \sum_{s_0,s_{\tau+1}}
    \frac{|s_{\tau+1}\rangle p(s_{\tau+1},s_0,\bm y_{:\tau}|\bm a_{:\tau}) \langle s_0 |}
    {p(\bm y_{:\tau}|\bm a_{:\tau})} 
    =
    \sum_{s_0,s_{\tau+1}}
    \frac{|s_{\tau+1}\rangle T^{(\bm y_{:\tau}|\bm a_{:\tau})} \rho_0 \langle s_0 | }
    {\langle 1 | T^{(\bm y_{:\tau}|\bm a_{:\tau})}  \rho_0 |1 \rangle}~,
\end{align}
which proves the first part of the theorem. 
Additionally, by comparing this with \autoref{eq:pred_quantum} and \autoref{eq:retrodic_quantum} one finds that
\begin{equation}
    \rho(\bm y_{:\tau},\bm a_{:\tau}) |1\rangle = |\rho^P(\bm y_{:t}, \bm a_{:t}) \rangle
    \quad\text{and}\quad
     \langle 1| \rho(\bm y_{:\tau},\bm a_{:\tau}) = \langle \rho^R(\bm y_{:t}, \bm a_{:t})| ~,
\end{equation}
which proves the second part of the theorem.

The corollary can be proven by noticing that
\begin{align}
\rho(\bm y_{:\tau+1},\bm a_{:\tau+1})
=
\frac{T^{(y_{\tau+1}|a_{\tau+1})}
\rho(\bm y_{:\tau},\bm a_{:\tau})}{\langle 1| T^{(y_{\tau+1}|a_{\tau+1})}\rho(\bm y_{:\tau},\bm a_{:\tau})|1\rangle},
\end{align}
where the denominator is a normalisation term. 
By contrast, the reverse-time update requires applying a modified version of the transducer operator $\rho_{0}^{-1} T^{(y|a)}\rho_0$ and normalizing:
\begin{align}
\rho(\bm y_{-1:\tau},\bm a_{-1:\tau})
=
\frac{\rho(\bm y_{:\tau},\bm a_{:\tau})\rho^{-1}_0T^{(y_{-1}|a_{-1})}\rho_{-1}}{\langle 1| \rho(\bm y_{:\tau},\bm a_{:\tau})\rho_0^{-1}T^{(y_{-1}|a_{-1})}\rho_{-1}|1\rangle}.
\end{align}
Reflecting the fact that not every transducer is reversible, the operation of $\rho_{0}^{-1} T^{(y|a)}\rho_0$ cannot necessarily be interpreted as the action of a transducer. However, it is nevertheless a valid method for retrodicting the state distribution of the world.
\end{proof}

It is important to note that, since not every transducer is reversible, the operation $\rho_{t}^{-1} T^{(y|a)}\rho_{t-1}$ generally does not yield the action of a transducer. This operation is, nevertheless, a valid method for retrodicting the state distribution of a world model if its initial state is assumed to be uncorrelated with future action sequences.

\end{document}